\documentclass[pdflatex,sn-mathphys-num]{sn-jnl}

\usepackage{graphicx}%
\usepackage{multirow}%
\usepackage{amsmath,amssymb,amsfonts}%
\usepackage{amsthm}%
\usepackage{mathrsfs}%
\usepackage[title]{appendix}%
\usepackage{xcolor}%
\usepackage{textcomp}%
\usepackage{manyfoot}%
\usepackage{booktabs}%
\usepackage{algorithm}%
\usepackage{algorithmicx}%
\usepackage{algpseudocode}%
\usepackage{listings}%
\usepackage{mathrsfs}

\theoremstyle{thmstyleone}%
\newtheorem{theorem}{Theorem}
\newtheorem{proposition}[theorem]{Proposition}%

\theoremstyle{thmstyletwo}%

\theoremstyle{thmstylethree}%

\usepackage{amsthm}
\usepackage{amsmath}
\usepackage{threeparttable}
\usepackage{graphicx}
\usepackage{subcaption}

\newcommand{\R}{\mathbb{R}}
\newcommand{\N}{\mathbb{N}}
\newcommand{\E}{\mathbb{E}}
\newcommand{\bL}{\mathbb{L}}
\renewcommand{\P}{\mathbb{P}}

\newcommand{\CF}{\mathcal{F}}
\newcommand{\CW}{\mathcal{W}}

\newcommand{\cL}{\mathcal{L}}
\newcommand{\CM}{\mathcal{M}}
\newcommand{\dd}{\text{d}}
\newcommand{\cT}{\mathcal{T}}
\newcommand{\cH}{\mathcal{H}}

\newtheorem{lemma}[theorem]{Lemma}

\newtheorem{assumption}[theorem]{Assumption}

\begin{document}

\title[Article Title]{When Langevin Monte Carlo Meets Randomization: New Sampling Algorithms with Non-asymptotic Error Bounds beyond Log-Concavity and Gradient Lipschitzness}


\author[1]{\fnm{Xiaojie} \sur{Wang}}\email{x.j.wang7@csu.edu.cn}

\author*[1]{\fnm{Bin} \sur{Yang}}\email{b.yang1@csu.edu.cn}


\affil[1]{\orgdiv{School of Mathematics and Statistics}, \orgname{HNP-LAMA, Central South University},
\\
\orgaddress{ \city{Changsha}, \postcode{410083}, \state{Hunan},
\country{China}}}





\abstract{Efficient sampling from complex and high dimensional target distributions turns out to be a fundamental task in diverse disciplines such as scientific computing, statistics and machine learning. In this paper, we propose a new kind of randomized splitting Langevin Monte Carlo (RSLMC) algorithm for sampling from high dimensional distributions without log-concavity. Compared with the existing randomized Langevin Monte Carlo (RLMC), the newly proposed RSLMC algorithm requires less evaluations of gradients and is thus computationally cheaper. Under the gradient Lipschitz condition and the log-Sobolev inequality, we prove a uniform-in-time error bound in $\mathcal{W}_2$-distance of order $O(\sqrt{d}h)$ for both RLMC and RSLMC sampling algorithms, which matches the best one in the literature under the log-concavity condition. Moreover, when the gradient of the potential $U$ is non-globally Lipschitz with superlinear growth, new modified R(S)LMC algorithms are introduced and analyzed, with non-asymptotic error bounds established. Numerical examples are finally reported to corroborate the theoretical findings. 
}

\keywords{Langevin dynamics, splitting scheme, randomized Langevin Monte Carlo, non-asymptotic error bound, log-Sobolev inequality}


\pacs[MSC Classification]{60H35, 60H10, 65C30}

\maketitle

\section{Introduction}
\label{Randomization-sec:introduction}

Sampling from a high dimensional target distribution 
$
\pi(\dd x) \propto \exp (-U( x)) \, \mathrm{d} x,
    \,
    x \in \mathbb{R}^d,
    \,
    d \gg 1
$ 
becomes a core problem in many research areas of scientific computing, statistics and machine learning \cite{liu2001monte,robert1999monte}. 
A prominent approach to this problem is the Langevin type sampling algorithm, which has been extensively studied in the literature.
The key idea of the Langevin sampling algorithm is to construct a Markov chain based on a time discretization of the continuous-time (overdamped) Langevin diffusion:
\begin{equation}
\label{Randomization-eq:langevin_SDE}
\dd X_{t}
=
-\nabla U(X_{t})
\, \dd t
+
\sqrt{2}
\, \dd W_{t},
\quad  
X_{0}=x_{0},
\quad 
t>0,
\end{equation}
where $W_\cdot :=\left(W^{1}_\cdot, W^{2}_\cdot, \cdots, W^{d}_\cdot\right)^T:[0, \infty) \times \Omega_W \rightarrow \mathbb{R}^d$ is a $d$-dimensional Brownian motion defined on a filtered probability space $(\Omega_W,\mathcal{F}^W,$\ $\{\mathcal{F}^W_t\}_{t\geq 0},\mathbb{P}_W)$, satisfying the usual conditions. The initial data $x_0: \Omega_W \rightarrow \mathbb{R}^d$ is assumed to be $\mathcal{F}^W_0$-measurable. Under mild conditions, the Langevin stochastic differential equation (SDE) admits the target distribution $\pi(\dd x) \propto \exp (-U( x)) \, \mathrm{d} x$ as its unique invariant distribution (see, e.g., \cite{pavliotis2014stochastic}). 
%
Therefore, one can turn sampling from the target distribution into long-time approximations of the Langevin SDE.
For a uniform timestep $h>0$, a popular choice of the discretization scheme for \eqref{Randomization-eq:langevin_SDE} is the Euler-Maruyama method defined by
\begin{equation}\label{Randomization-eq:intro-LMC}
\hat{Y}_{n+1} 
=
\hat{Y}_{n}  
-   
\nabla U  (  \hat{Y}_{n}  ) h
+
\sqrt{2 h }    \zeta_{n+1},
\quad 
\hat{Y}_0  =  x_0,
\end{equation}
where 
$
\zeta_{n}
:= 
(  \zeta_{n}^1,  
   \zeta_{n}^2, 
    \cdots, 
    \zeta_{n}^d
)^T,
$
$
n \in \N
$,
are i.i.d standard
$d$-dimensional Gaussian variables. Such an algorithm, usually termed as unadjusted Langevin algorithm (ULA) or the Langevin Monte Carlo (LMC), has been extensively investigated over recent years, with a particular focus on the non-asymptotic error analysis. Next we start from presenting some related works on this topic.

\subsection{Randomized sampling algorithms: RLMC and RSLMC}

The early non-asymptotic error analysis of LMC was carried out under a strongly log-concave condition ($m>0$):
\begin{equation}
\label{Randomization-eq:log-concave_condition}
\langle
  x-y,
  \nabla U\left(x\right)
  -
  \nabla U\left(y\right)
\rangle 
\geq 
 m |x-y|^2, 
\quad 
\forall x,y \in \mathbb{R}^d,
\end{equation}
that is, the potential
$U$ is strongly convex (see, e.g., \cite{cheng2018convergence,chewi2021optimal,dalalyan2017further,dalalyan2017theoretical,dalalyan2019user,durmus2019high,li2022sqrt,li2019stochastic,roberts1996exponential,sabanis2019higher,vempala2019rapid}, to just mention a few).
As shown in \cite{cheng2018convergence,dalalyan2017further,dalalyan2017theoretical,durmus2019analysis,Durmus2017Non},
the non-asymptotic convergence of order $O(\sqrt{dh})$ can be obtained for LMC under a gradient Lipschitz condition:
\begin{equation}
\label{Randomization-eq:intro-gradient-Lipschitz}
\big|
\nabla U
(x) 
-
\nabla U
(y) 
\big| 
\leq 
L_1  |  x  -  y|,
\quad 
\forall x,y \in \mathbb{R}^d.
\end{equation}
%
In order to  attain the first order convergence, the price to pay is usually putting additional smoothness assumption on the potential $U$. Indeed, by additionally imposing the Hessian Lipschitzness condition:
\begin{equation}
\label{Randomization-eq:intro-Hessian-Lipschitz}
|| 
\nabla^2 U (x) 
-
\nabla^2 U (x) 
||
\leq  
\bL 
|x-y|,
\quad 
\forall 
x,y\in \R^d,
\end{equation}
Durmus and Moulines \cite{durmus2019high} proved an improved error bound $O(dh)$ in $\mathcal{W}_2$-distance for the LMC.
Under the linear growth condition of the $3$-rd derivative of $U$:
\begin{equation}
\label{Randomization-eq:intro-Linear_growth_of_the_3rd-order_derivative}
   |
      \nabla(\Delta U\left(x\right))
   |
\leq 
   L_0'd^{\frac12}   
   +   
   L_0 | x  |,
\quad 
\forall x \in \mathbb{R}^d,
\end{equation}
instead of the Hessian Lipschitzness condition \eqref{Randomization-eq:intro-Hessian-Lipschitz},  Li et al. \cite{li2022sqrt} derived a further improved error bound $O(\sqrt{d}h)$ under the strongly log-concave condition.
An interesting question is whether there is any sampling algorithm that has a non-asymptotic error bound $O(\sqrt{d}h)$, without requiring additional smoothness assumptions on the potential $U$ other than the gradient Lipschitz condition.
Before answering this question, let us recall an exiting randomized Langevin Monte Carlo (RLMC):
\begin{equation}
\label{Randomization-eq:RLMC}
\begin{aligned}
Y_{n+1}^\tau
=  &
Y_n 
-
\nabla U (
Y_n 
)
\tau_{n+1} h 
+ 
\sqrt{2} 
\Delta W_{n+1}^\tau,
\quad 
Y_0=x_0,
\\
Y_{n+1}
=  &
Y_n  
-
\nabla U (Y_{n+1}^\tau ) h 
+ 
\sqrt{2} 
\Delta W_{n+1},
\end{aligned}
\end{equation}
where $\{\tau_n\}_{n\in\N}$ is an i.i.d family of  uniform distribution on the interval $(0,1)$ ($\mathcal{U}(0,1)$ in  short) defined on an additional filtered probability space $(\Omega_\tau, \mathcal{F}^\tau, \{\mathcal{F}^\tau_n\}_{n\in \N}, \P_\tau)$ with $\mathcal{F}^\tau_n$ being the $\sigma$-algebra generated by $\{\tau_n\}_{n\in \N}$, 
$\Delta W_{n+1}^\tau
:=
W_{t_n+\tau_{n+1}h}
-
W_{t_n}$  and 
$\Delta W_{n+1}
:=
W_{t_{n+1}}
-
W_{t_n}$.
Here the random variables $\{\tau_n\}_{n\in\N}$ are artificially added random inputs, which are assumed to be independent
of the randomness already presented in Langevin SDE \eqref{Randomization-eq:langevin_SDE}.
We mention that such a randomized method was first introduced for ordinary differential equations (ODEs) with irregular coefficients \cite{DAUN2011300,heinrich2008randomized,jentzen2009random,Stengle1990Numerical,Stengle1995Error}, and was further extended to SDEs with irregular coefficients \cite{kruse2017error,kruse2019randomized,przybylowicz2014strong} in recent years. In 2019, the idea of  randomization was introduced for the LMC sampling \cite{shen2019randomized}.
As shown by \cite{shen2019randomized,yu2024langevin}, RLMC exhibits better performance that the classical LMC in terms of both tolerance and condition number dependency under the log-concavity condition. In particular, the non-asymptotic error bound $O(\sqrt{d}h)$ can be achieved for RLMC just under the gradient Lipschitz condition, without requiring
additional smoothness conditions on $\nabla U$. This thus gives a positive answer to the aforementioned question.

To implement the scheme \eqref{Randomization-eq:RLMC}, one needs to evaluate $\nabla U$ twice for every time step.
In many applications such as sampling from the posterior distribution in Bayesian inverse problem \cite{stuart2010inverse}, the evaluation of $\nabla U$ would be computationally expensive.
%
%
A natural and interesting question thus emerges: 
\vspace{0.2cm}

{\bf (Q1).} {\it
Can one construct a new randomized LMC with only one evaluation of $\nabla U$ per one time step, still preserving the same convergence rate as the existing RLMC?
}
\vspace{0.2cm}

Following this direction and inspired by a randomized splitting idea, we introduce a new randomized splitting LMC (RSLMC) algorithm with three splitting steps, given by 
\begin{equation} \label{eq:RSLMC0}
\begin{aligned}
\mathbb{Y}_{n+1}^{(1)}
=  &
\mathbb{Y}_n 
+ 
\sqrt{2} 
(W_{t_n+\tau_{n+1}h}
-
W_{t_n}),
\quad
\mathbb{Y}_0=x_0,
\\
\mathbb{Y}_{n+1}^{(2)}
= & 
\mathbb{Y}_{n+1}^{(1)}
-
\nabla U ( \mathbb{Y}_{n+1}^{(1)} ) h,
\\ 
\mathbb{Y}_{n+1}
=   & 
\mathbb{Y}_{n+1}^{(2)}
+
\sqrt{2} 
(
W_{t_{n+1}}
-
W_{t_n+\tau_{n+1}h}
).
\end{aligned}
\end{equation}
Equivalently, the new algorithm \eqref{eq:RSLMC0} can be rewritten in a compact form of
\begin{equation}
\label{Randomization-eq:RSLMC}
\begin{aligned}
\mathbb{Y}_{n+1}^\tau
=  &
\mathbb{Y}_n 
+ 
\sqrt{2} 
\Delta W_{n+1}^\tau,
\quad 
\mathbb{Y}_0=x_0,
\\
\mathbb{Y}_{n+1}
=  &
\mathbb{Y}_n  
-
\nabla U (\mathbb{Y}_{n+1}^\tau ) h 
+ 
\sqrt{2} 
\Delta W_{n+1}.
\end{aligned}
\end{equation}
Clearly, the RSLMC algorithm involves only one evaluation of $\nabla U$ per one time step, which is computationally cheaper than the existing RLMC \eqref{Randomization-eq:RLMC}. Next we show that the new RSLMC algorithm still preserves the same convergence rate as the existing RLMC \eqref{Randomization-eq:RLMC}.

Recall that the above non-asymptotic error bounds we just mentioned are all obtained under the strongly log-concave condition, which is, however, extremely restrictive and seldom satisfied in practice. Without the log-concavity condition, the corresponding non-asymptotic error analysis turns out to be a challenging task (see, e.g., \cite{cheng2018sharpconvergencerateslangevin,chewi2024analysis,li2025unadjusted,lytras2024tamed,lytras2025taming,majka2020non,Mou2022Improved,neufeld2025non,pages2023unadjusted,pang2023projected,vempala2019rapid}).
When the potential $U$ is strongly-convex outside a ball but possibly nonconvex inside this ball, Cheng et al  \cite{cheng2018sharpconvergencerateslangevin}
established an upper bound
$O(\sqrt{dh})$ for LMC in $\CW_1$-distance.
Under convexity at infinity condition, 
Majka et al \cite{majka2020non}
showed error bounds $O (\sqrt[4]{dh} )$ and $O(\sqrt{dh})$ in $\mathcal{W}_2$- and $\mathcal{W}_1$-distance, respectively.
Later on, Mou et al \cite{Mou2022Improved} 
obtained improved Kullback-Leibler divergence bounds, implying an error bound $ O ( d h )$ in both total variation distance and $\mathcal{W}_2$-distance, under smoothness conditions on $U$ including the Hessian Lipschitzness condition \eqref{Randomization-eq:intro-Hessian-Lipschitz} and the assumption that the target distribution satisfies a log-Sobolev inequality (LSI).
Very recently, Yang and Wang \cite{yang2025nonasymptotic} and Altschuler and Chew \cite{Altschuler2024ShiftedCI} proved an  error bound $O(\sqrt{d} h)$ in $\mathcal{W}_2$-distance for the classical LMC, also in the framework of LSI, under smoothness conditions on $U$ including the linear growth condition \eqref{Randomization-eq:intro-Linear_growth_of_the_3rd-order_derivative} of the $3$-rd derivative of $U$.



Bearing these error bounds in mind, we want to raise the following interesting question:
\vspace{0.2cm}

{\bf (Q2).} {\it Beyond log-concavity, can the non-asymptotic error bound $O(\sqrt{d}h )$ in $\mathcal{W}_2$-distance still hold true for both RLMC and RSLMC, under no additional smoothness conditions other than the gradient Lipschitz condition?}
\vspace{0.2cm}

As the first main contribution of this paper, we attempt to answer (Q1) and (Q2) to the positive. More precisely, in a non-convex setting of LSI, we derive the uniform-in-time error bound $O(\sqrt{d} h)$ in $\mathcal{W}_2$-distance for both RLMC and RSLMC, without additional smoothness assumptions on the potential $U$ other than the gradient Lipschitz condition (see Theorem \ref{thm:main_thm_for_RLMC}). 

But for some potentials like the double-well potential
$U(x) = \tfrac{\alpha}{4} |x|^4 -\tfrac{\beta}{2} |x|^2
$, 
the gradient Lipschitz condition is violated. 
So one might ask the following question:
\vspace{0.2cm}
%
%
%

{\bf (Q3).} {\it What if the gradient Lipschitz condition is violated?}
\vspace{0.2cm}

Following this question, we continue to examine the sampling problem when $\nabla U$ is non-globally Lipschitz with superlinear growth. As shown by \cite{Hutzenthaler2011Strong},
the usual Euler discretization scheme \eqref{Randomization-eq:intro-LMC} for such SDEs (i.e., LMC) over finite time fails to be convergent. To remedy it, we introduce a modified R(S)LMC \eqref{Randomization-eq:pRLMC} and carefully analyze its uniform-in-time error bound, with the dimension dependence revealed (see Theorem \ref{thm:main_thm_for_pRLMC} and its proof). We would like to mention that some tamed Langevin sampling algorithms without randomization were also proposed and analyzed in \cite{lytras2024tamed,lytras2025taming,neufeld2025non,sabanis2019higher} under non-globally Lipschitz conditions.

The long-time error analysis for both gradient Lipschitz and non-globally Lipschitz settings follows a unified approach, essentially relying on the exponential ergodicity in $\mathcal{W}_2$-distance of the Langevin SDE, as presented in Proposition \ref{Randomization-prop:exponential ergodicity}. More accurately, by a local error analysis (see Lemma \ref{lem:onestep_waek_and_strong_error_of_RLMC}), we first establish finite-time mean-square convergence rates of the sampling algorithms, suffering from exponential time dependence (see Propositions \ref{prop:finite_time_error_analysis_RLMC} and \ref{prop:finite_time_error_analysis_pRLMC}).
This combined with the exponential ergodicity in $\mathcal{W}_2$-distance of the Langevin SDE and uniform-in-time moment bounds of the algorithms enables us to obtain uniform-in-time error bounds, without suffering from exponential time dependence.
As indicated by Lemma \ref{lem:onestep_waek_and_strong_error_of_RLMC}, the R(S)LMC has a local mean error of order $O(\sqrt{d} h^{2})$ and a local mean-square error of order $O(\sqrt{d} h^{\frac32})$ under only the gradient Lipschitzness condition. This analysis sheds new light on how R(S)LMC works well.
For more details, please consult subsection \ref{Randomization-subsec:technical_overview} and proofs of the main results.

After the present work was finished, we were informed about the very interesting paper \cite{Altschuler2024ShiftedCI}, where a shifted composition rule was used to set up a local error framework for KL divergence, which provided a unified error analysis in KL divergence for both LMC \eqref{Randomization-eq:intro-LMC} and RLMC \eqref{Randomization-eq:RLMC} algorithms under the gradient Lipschitz condition. In particular, an error bound $O(\sqrt{d}h )$ in $\mathcal{W}_2$-distance can be obtained for RLMC in a LSI and gradient Lipschitz setting (cf. \cite[Theorem 6.4]{Altschuler2024ShiftedCI}), where the potential function $U$ was, however, additionally required to be twice continuously differentiable. 
%
%
Also, we mention that their approach does not rely on the dissipativity condition \eqref{Randomization-eq:diss_cond}, but fails to work for the setting of non-globally Lipschitz $\nabla U$ (see Lemmas C.3 and C.4 in \cite{{Altschuler2024ShiftedCI}}). 
Instead, the error bound $O(\sqrt{d}h )$ is derived here for RLMC under no additional smoothness conditions other than the gradient Lipschitz condition and our methodology also works well for non-globally Lipschitz $\nabla U$ with super-linear growth, as already mentioned above.

\subsection{Contributions of this work}
\label{Randomization-subsec:Contributions}

In summary, the main contribution of this paper is as follows:
\begin{itemize}
    \item A new randomized splitting Langevin Monte Carlo (RSLMC) sampling algorithm is proposed, only requiring only one evaluation of $\nabla U$ per one time step and improving the existing randomized Langevin Monte Carlo (RLMC) in terms of the number of evaluations of $\nabla U$.
    \item A novel approach of uniform-in-time error analysis in $\mathcal{W}_2$-distance is introduced for the aforementioned two randomized sampling algorithms, which works for both the gradient Lipschitz case and the case when the gradient of the potential $U$ is non-Lipschitz with superlinear growth.
    \item When the target distribution satisfies a log-Sobolev inequality, 
    an error bound $O(\sqrt{d} h)$ in $\mathcal{W}_2$-distance is derived for the existing RLMC and the new RSLMC, without additional smoothness assumptions on the potential $U$ other than the gradient Lipschitz condition. This bound matches the best one in the strongly log-concave case and improves upon the best-known convergence results in non-convex settings. 
    \item For the case when the gradient of the potential $U$ is non-globally Lipschitz with superlinear growth, novel modified R(S)LMC sampling algorithms are proposed and analyzed, with an non-asymptotic error bound in $\mathcal{W}_2$-distance explicitly shown in the non-convex setting.
\end{itemize}

This paper is organized as follows. The next section collects some notations throughout this paper and establishes the exponential ergodicity of the Langevin dynamics under LSI. 
Section \ref{Randomization-eq:main_result} presents main results for the RLMC and modified RLMC sampling algorithms. 
Proofs of main results are put in Sections \ref{Randomization-sec:app:proofs_of_Lip_case}, \ref{Randomization-sec:app:proofs_of_non_Lip_case}. Numerical experiments are provided in Section \ref{Randomization-sec:Numerical experiments}
%
%
and some concluding remarks are given in the last section.

\section{Exponential ergodicity of the Langevin dynamics without log-concavity}
\label{Randomization-sec:Exponential_ergodicity_of_LSDE}

The focus of this section is to show exponential ergodicity of Langevin dynamics in $\CW_2$-distance without the commonly used log-concavity condition.

\subsection{Notation}
\label{Randomization-subsec:notation}


\textbf{Notation.} 
Throughout this paper,
we denote by $\N $ the set of all positive integers and let $\N_0:=\N \cup \{0\}$. For all $n \in \N$, we denote $[n] := \{1,2,\cdots, n\}$ and $[n]_0 := \{0,1,\cdots, n\}$.
For convention, we set $0^0=1$.
The symbols $\wedge$ and $\vee$ are used to mean “minimum” and “maximum”, respectively.
For simplicity, we write $\Tilde{O}(\cdot)$ to mean $O(\cdot)\log^{O(1)}(\cdot)$.
Also, we use the notation
$\langle \cdot, \cdot \rangle $ and $ |\cdot|$
to denote the inner product and the Euclidean norm of vectors in $  \R^d$, respectively.
Let
$\| \cdot \|$ denote the operator   norm  of matrices.  
For a function 
$f:\R^d \rightarrow \R$, we   write $\partial_i f $ to denote the $i$-th partial derivative of $f$. The gradient $\nabla f$ is the vector of partial derivatives $(\partial_1 f,\cdots,\partial_d f)^T$ and the Hessian $\nabla^2 f$ is the
matrix $(\partial^2_{ij}f)_{i,j\in[d]}$.  
The Laplacian of $f$ is denoted by $\Delta f := \text{tr} \nabla^2 f = \sum_{i=1}^d  \partial^2_{ii}f$.

Let  $
\mathcal{B}
(\mathbb{R}^d
)$ be  the Borel $\sigma$-algebra of $\mathbb{R}^d$ and let  $\mathcal{P}(\R^d)$ be the space of all probability distributions on $ (
\R^d, \mathcal{B}(\R^d ))$.
For $\nu_1, \nu_2 \in \mathcal{P}(\R^d) $ We  define the $L^p$-Wasserstein distance (abbreviated as $\mathcal{W}_p$-distance) by
\begin{equation}
\mathcal{W}_p(\nu_1, \nu_2)
    :=
    \inf_{\varrho 
    \in \Gamma
    (\nu_1, \nu_2)}
    \left(\int_{\mathbb{R}^d \times \mathbb{R}^d}
    |x-y|^p 
    \dd \varrho(x, y)\right)^{1 / p},
\quad 
p \ge 1 ,
\end{equation}
where $\Gamma (
\nu_1, \nu_2
)$ denotes the set of all couplings of $\nu_1$  and $\nu_2$.

Denote by $C_b(\R^d)$ (resp. $B_b(\R^d) $) the Banach space of all uniformly continuous differentiable  and bounded mappings (resp. Borel bounded mappings). For $l\in \N$, let  $C^l_b(\R^d)$ be the subspace of $C_b(\R^d)$ consisting of all  $l$-times continuously differentiable functions with bounded partial derivatives. For any $f \in C_b(\R^d)$ and $\nu \in \mathcal{P}(\R^d)$, we define $\nu (f):= \int_{\R^d} f (x) \nu (dx)$.

Furthermore, for a given probability space $(\Tilde{\Omega},\Tilde{\CF},\Tilde{\P})$, we use $\Tilde{\E}$ to mean the expectation with respect to $\Tilde{\P}$, which is defined as $
\Tilde{\E} [X] 
:=
\int_{\Tilde{\Omega}} 
X(\omega) 
\Tilde{\P}( \dd \omega),
$ 
for any random variable $X:
\Tilde{\Omega} \rightarrow \R^d$.
Let $L^p(\Tilde{\Omega};\R^d)
$, $p\geq 1$, be the set consisting of all random variables $X:
\Tilde{\Omega} \rightarrow \R^d$ satisfying  
$
 \Tilde{\E} [|X(\omega)|^p] < \infty
$.
In this paper, we introduce a new product probability space $(\Omega, \CF, \P)$, generated by the Langevin SDE \eqref{Randomization-eq:langevin_SDE}, RLMC \eqref{Randomization-eq:RLMC} and RSLMC \eqref{Randomization-eq:RSLMC}, in form of 
\begin{equation}
\label{Randomization-eq:probability_space}
(\Omega, \CF, \P)
:=
(
\Omega_W \times  \Omega_{\tau},
\CF^W \otimes \CF^{\tau}
,
\P_W \otimes \P_{\tau}
).
\end{equation}
For the uniform stepsize $h>0$, we denote $t_n:=nh$ and define a discrete-time filtration $\{\CF_{t_n}\}_{n\in\N}$ on $(\Omega, \CF, \P)$ by 
\begin{equation}
\CF_{t_n} :=
\CF_{t_n}^W \otimes \CF_{n}^{\tau}, \quad 
\forall \, 
n \in \N.
\end{equation}
For the probability
space $(\Omega, \CF,\P)$ introduced by \eqref{Randomization-eq:probability_space}, let $\E$ denote the expectation and the Fubini theorem implies that, for any $X\in L^p(\Omega; \R^d)$,
\begin{equation}
\E [X]
=
\E_W 
\big[
\E_{\tau}
[X]
\big]
=
\E_{\tau}
\big[
\E_W 
[X]
\big],
\end{equation}
where $\E_W$ is the expectation with respect to $\P_W$ and $\E_{\tau}$ with respect to $\P_\tau$.
Let $X_t:=X(s,x;t)=X_{s,x}(t)$, $0\leq s \leq t$, be the solution to \eqref{Randomization-eq:langevin_SDE} at time $t$ with the initial value $x$ at $s$, which is given by 
\begin{equation}
\label{Randomization-eq:solutions_of_LSDE}
X(s,x;t)
=
x
-
\int_{s}^t
\nabla U 
(X(s,x;r))
\,
\dd r 
+
\int_{s}^t
\sqrt{2} \,
\dd W_r.
\end{equation}

\subsection{Exponential ergodicity under a log-Sobolev inequality}
\label{Randomization-subsec:Exponential_ergodicity_for_LSDE}

This subsection aims to establish exponential ergodicity in $\CW_2$-distance of the Langevin dynamics \eqref{Randomization-eq:langevin_SDE} in the setting of LSI. We first place a dissipativity condition as follows.

\begin{assumption}
\label{Randomization-ass:diss_cond_and_grad_Lip_cond}
The drift function  
$-\nabla U$  of the Langevin dynamics \eqref{Randomization-eq:langevin_SDE} satisfies a
dissipativity condition, i.e., there exist two constants 
$\mu, \mu' > 0 $, independent of dimension $d$, such that for any 
$ x \in \R^d$,
\begin{equation}
\label{Randomization-eq:diss_cond}
\big\langle 
x, 
\nabla U
(x)
\big\rangle 
\geq 
\mu
|x|^2
-
\mu'd;
\end{equation}
\end{assumption}
As commonly used for a non-convex setting in the literature, e.g., \cite{Mou2022Improved,majka2020non,pages2023unadjusted,lytras2025taming,xu2018global}, the dissipativity condition  \eqref{Randomization-eq:diss_cond} was required here to prove uniformly bounded moments of the Langevin diffusion and its discretizations.
Such a condition is indeed not a smoothness condition and can be further relaxed to a Lyapunov type one \cite{pages2023unadjusted}.

\begin{assumption}
\label{Randomization-ass:one-sided-lip}
The drift function  
$-\nabla U$  of the Langevin dynamics \eqref{Randomization-eq:langevin_SDE} satisfies
one-sided Lipschitz condition, i.e., there exists a dimension-independent constant $L>0$ such that for any $x,y \in \R^d$,
\begin{equation}
\label{Randomization-eq:one_sided_Lip_cond}
\big\langle 
x-y,
\nabla U
(x) 
-
\nabla U
(y)
\big\rangle
\geq 
- L  |  x  -  y|^2.
\end{equation}
\end{assumption}

%
%

\begin{assumption}
\label{Randomization-ass:log_Sobolev_inequality}
Let $\{p_t\}_{t \geq 0}$ be the semigroups associated to the continuous-time Langevin diffusion \eqref{Randomization-eq:langevin_SDE} admitting a unique invariant distribution $\pi$.
%
%
For any $f \in C_b^1(\R^d)$, there exists a dimension-independent constant $\rho$   such that the invariant distribution $\pi$ satisfies the log-Sobolev inequality (LSI):
\begin{equation}
\pi
(f^2 \log f^2) 
\leq 
\rho
\pi 
(|\nabla f|^2),
\quad 
\pi (f^2) =1.
\end{equation}
\end{assumption}

In the context of sampling, the LSI is a widely used condition in  non-strongly convex settings.
As shown in supplementary material of \cite{Mou2022Improved},
this assumption is weaker than strongly convex outside of a ball used in \cite{cheng2018sharpconvergencerateslangevin,neufeld2025non}.
Indeed, LSI is the most well-studied functional inequality for the target
distribution of interest in the study of Langevin sampling
\cite{chewi2024analysis,lytras2024tamed,lytras2025taming,Mou2022Improved,vempala2019rapid}.

\begin{proposition}[Uniformly bounded moments]
\label{prop:uniform_in_time_bounded_moments_to_LSDE}
Let Assumption \ref{Randomization-ass:diss_cond_and_grad_Lip_cond} hold.
Let $\{X_t\}_{t\geq 0}$ be the solution of the Langevin SDE \eqref{Randomization-eq:langevin_SDE}. Then for any 
$p \in [1,\infty)$ it holds
\begin{equation}
\label{Randomization-eq:uniform_in_time_bounded_moments_to_LSDE}
\sup_{t\geq0}
\E_W
\big[ 
| X_t |^{2p}
\big]
\leq 
e^{-cpt}
\E_W
\big[ 
| x_0|^{2p}
\big]
+
\CM_1(p) d^p,
\end{equation}
where $c\in(0,2 \mu)$ and  $\CM_1(p):=\frac{2(2p-1+\mu')^p}{cp}
(
  \frac{2p-2}{(2\mu-c)p}
)^{p-1}$ are independent of $d$ and $t$.
\end{proposition}
The proof of this proposition can be found in \cite[Lemma 2.4]{yang2025nonasymptotic}. 
Next, we present a proposition on exponential ergodicity in $\CW_2$-distance of the Langevin SDE \eqref{Randomization-eq:langevin_SDE} in the LSI setting, which can be found in \cite[Theorem 2.1 (2) and Theorem 2.6 (2)]{Wang2020Exponential} and \cite[Proposition 2.5]{yang2025nonasymptotic}.

\begin{proposition}[Exponential ergodicity in $\mathcal{W}_2$-distance]
\label{Randomization-prop:exponential ergodicity}
\label{pro:exponential_convergence}
Let Assumptions \ref{Randomization-ass:one-sided-lip} and  \ref{Randomization-ass:log_Sobolev_inequality} be satisfied.
Then for any  initial distribution $\nu := \cL (x_0) $, the transition semigroup $p_t$  and its invariant distribution $\pi$ satisfy 
\begin{equation}
\mathcal{W}_2 (\nu p_t, \pi) 
\leq
\mathcal{K}
e^{- \eta t}
\mathcal{W}_2 (\nu ,  \pi),
\quad
\forall \ t \geq 0,
\end{equation}
where $\mathcal{K}=(
\tfrac{2 \rho L}{1-e^{-2L}})^{\frac12}
e^{\frac{4}{\rho}}
\vee 
e^{2 L+\frac{2}{\rho}}$ and 
$\eta= \frac{2}{\rho}$.
\end{proposition}


\section{Main results}
\label{Randomization-eq:main_result}
In this section we present main results for the considered sampling algorithms.

\subsection{Non-asymptotic error bounds for randomized (splitting) Langevin Monte Carlo}
\label{Randomization-subsec:main_result_for_RLMC}

In this subsection we restrict ourselves to non-asymptotic error bounds of RLMC \eqref{Randomization-eq:RLMC}, RSLMC \eqref{Randomization-eq:RSLMC} in $\CW_2$-distance without log-concavity.
We first put a gradient Lipschitz condition on the drift.

\begin{assumption}
\label{Randomization-ass:Lip_condition}
The drift function  
$-\nabla U$ of Langevin dynamics \eqref{Randomization-eq:langevin_SDE} satisfies a gradient Lipschitz condition, i.e., there exists a dimension-independent constant $L_1>0$ such that for any $x,y \in \R^d$,
\begin{equation}
\label{Randomization-eq:Lip_cond}
\big|
\nabla U
(x) 
-
\nabla U
(y) 
\big| 
\leq 
L_1  |  x  -  y|.
\end{equation}
\end{assumption}

The gradient Lipschitz condition  ensures 
\begin{equation}
\label{Randomization-eq:linear_growth_cond}
|\nabla U (x)  | 
\leq 
|\nabla U (0)|
+
L_1  |  x |
\leq
L_1' d^{\frac12}
+
L_1  |  x |,
\quad 
\forall \, x \in \R^d.
\end{equation}
Under the gradient Lipschitz condition, the one-sided Lipschitz condition \eqref{Randomization-eq:one_sided_Lip_cond} holds with $L=L_1$.
\textcolor{black}{To carry out the non-asymptotic error analysis for RLMC \eqref{Randomization-eq:RLMC} and RSLMC \eqref{Randomization-eq:RSLMC} in a unified way, we introduce a unified process $Y_{n+1}^\ell$ for $\ell \in \{0,1\}$, defined by
\begin{equation}
\label{Randomization-eq:unified-process}
\begin{aligned}
Y_{n+1}^{\tau,\ell}
=  &
Y_n^\ell
-
\ell \nabla U (
Y_n^\ell 
)
\tau_{n+1} h 
+ 
\sqrt{2} 
\Delta W_{n+1}^\tau,
\quad 
Y^\ell_0=x_0,
\\
Y_{n+1}^{\ell}
=  &
Y_n^\ell  
-
\nabla U (Y_{n+1}^{\tau,\ell} ) h 
+ 
\sqrt{2} 
\Delta W_{n+1},
\quad 
n  \in  \N_0,
\end{aligned}
\end{equation}
which reduces to the RLMC scheme \eqref{Randomization-eq:RLMC} for $\ell = 1$ (i.e., $Y_n^1 = Y_n, Y_{n+1}^{\tau,1} = Y_{n+1}^{\tau}$) and reduces to the RSLMC scheme \eqref{Randomization-eq:RSLMC} for $\ell = 0$ (i.e., $Y_n^0 = \mathbb{Y}_n, Y_{n+1}^{\tau,0} = \mathbb{Y}_{n+1}^{\tau}$). 
}
%
One of the key elements for the analysis of the non-asymptotic error bound is to establish the uniform-in-time bounded moments of the RLMC algorithm \eqref{Randomization-eq:RLMC} and  RSLMC algorithm \eqref{Randomization-eq:RSLMC}.

\begin{proposition}[Uniform-in-time moment bounds of R(S)LMC]
\label{prop:uniform_in_time_bounded_moments_to_RLMC}
Let Assumptions \ref{Randomization-ass:diss_cond_and_grad_Lip_cond}, \ref{Randomization-ass:Lip_condition} hold.
If the uniform stepsize  $h$ satisfies $h \leq 
1 \wedge \tfrac{1
}{\mu} \wedge \tfrac{1}{L_1}\wedge\tfrac{1}{L_1'} \wedge \textcolor{black}{\tfrac{\mu}{(5+16\ell^2)L_1^2}}$,  then it holds, for all $n \in \N_0$
\begin{equation} 
\label{Randomization-eq:uniform_in_time_bounded_moments_to_RLMC}
\E 
\big[ 
| Y_n^{\textcolor{black}{\ell}} |^{2}
\big]
\leq 
e^{-\mu t_{n}}
    \E
\big[
| x_0|^2
\big]
+
\CM_2^{{\textcolor{black}{\ell}}} 
d ,
\end{equation}
where \textcolor{black}{$\{Y^\ell_n\}_{n\in\N_0}$ is given by \eqref{Randomization-eq:unified-process} and}
$\CM^{\textcolor{black}{\ell}}_2:=\frac{(
20 + 
\textcolor{black}{(4+16\ell^2)} L_1'^2h + 2 \mu'
) }{\mu}$ for $\ell \in \{ 0, 1\}$.
\end{proposition}

The proof of Proposition \ref{prop:uniform_in_time_bounded_moments_to_RLMC} is postponed to subsection \ref{Randomization-sec:app:proofs_of_Lip_case}.
We now present the following finite-time convergence result of R(S)LMC.

\begin{proposition}[Finite-time convergence of R(S)LMC]
\label{prop:finite_time_error_analysis_RLMC}
Let  Assumptions  \ref{Randomization-ass:diss_cond_and_grad_Lip_cond}, \ref{Randomization-ass:Lip_condition}  hold. Let $\{X_t\}_{t\geq 0}$ and $\{Y_n^{\textcolor{black}{\ell}}\}_{n\in\N_0}$ be solutions of the Langevin SDE \eqref{Randomization-eq:langevin_SDE} and
\textcolor{black}{
the unified algorithm \eqref{Randomization-eq:unified-process},}
respectively. If the uniform stepsize $h \leq 1 \wedge \tfrac{1}{L_1}\wedge\tfrac{1}{L_1'} $, then for fixed $T=n_1h$, $n_1\in \N$ we have 
\begin{equation}
\sup_{n\in [n_1]}
\E
\big[
|
X_{t_n} 
-
Y_n^{\textcolor{black}{\ell}}
|^2
\big]
\leq 
\exp
\big( 1+12 L_1 T\big)
\big(
K_1^{\textcolor{black}{\ell}}  d
+
K_2^{\textcolor{black}{\ell}} 
\E
\big[
|
x_0
|^{2}
\big]
\big)  
h^2,
\end{equation}
where 
\begin{equation}
\begin{aligned}
K_1^{\ell}
:=  &
\textcolor{black}{
(8 0 + 40 \ell^2 )
L_1^2 
L_1'
+  
(4   +   4 \ell^2 ) 
L_1^2 
L_1'^2
+
80  \CM_1(1) L_1^3 
+
4  \CM_1(1) L_1^4  }
\\   & 
\textcolor{black}{
+ 
(8 0 + 40 \ell^2 )
\CM_2^{\ell} L_1^3  
+  
(4   +   4 \ell^2 )
\CM_2^{\ell} L_1^4
+
4 0  L_1^2 
\big),}
\\
K_2^{\ell}
:=  &
\textcolor{black}{
(8 0 + 40 \ell^2) 
L_1^3
+  
(4 + 4   \ell^2)  L_1^4
.}
\end{aligned}
\end{equation}
\end{proposition}

The proof of Proposition \ref{prop:finite_time_error_analysis_RLMC} can be found in subsection \ref{Randomization-sec:app:proofs_of_Lip_case}.
Thanks to  Propositions \ref{prop:uniform_in_time_bounded_moments_to_LSDE},
\ref{prop:uniform_in_time_bounded_moments_to_RLMC} and \ref{prop:finite_time_error_analysis_RLMC}, the non-asymptotic error bounds for RLMC \eqref{Randomization-eq:RLMC} and  RSLMC \eqref{Randomization-eq:RSLMC} can be obtained.

\begin{theorem}[Main result for R(S)LMC]
\label{thm:main_thm_for_RLMC}
Let Assumptions  \ref{Randomization-ass:diss_cond_and_grad_Lip_cond}, \ref{Randomization-ass:log_Sobolev_inequality} and \ref{Randomization-ass:Lip_condition} be satisfied.
Let $h$ be the uniform stepsize satisfying  $h \leq 
1 \wedge \tfrac{1
}{\mu} \wedge \tfrac{1}{L_1}\wedge\textcolor{black}{\tfrac{\mu}{(5+16\ell^2)L_1^2}}$ and let $q_n^{\textcolor{black}{\ell}}$ denote the transition probability of 
\textcolor{black}{the unified process \eqref{Randomization-eq:unified-process}}
at time $t_n:=nh$.
If there exists a dimension-independent constant $\sigma$ such that initial value $x_0\in \R^d$ satisfies 
$
\label{Randomization-eq:initial_value_condition1}
\E \big[   | x_0 |^2 \big]
\leq 
\sigma d,
$
then the law $\nu q_n^{\textcolor{black}{\ell}}$ of the $n$-th iterate $Y_n^{\textcolor{black}{\ell}}$ of \textcolor{black}{the unified process \eqref{Randomization-eq:unified-process}}
obeys 
\begin{equation}
\label{Randomization-eq:main_result_lip
}
\mathcal{W}_2 (\nu q_n^{\textcolor{black}{\ell}}, \pi )
\leq
C_1^{\textcolor{black}{\ell}} 
\sqrt{d} h  
+
C_2^{\textcolor{black}{\ell}} 
\sqrt{d} e^{-\lambda n h}
\end{equation}
for any $n\in \N$ and initial distribution $\nu=\mathcal{L}(x_0)$, where
\begin{equation}
\begin{aligned}
C_1^{\textcolor{black}{\ell}}
:=& 
\exp
\big( 1  + 12 L_1 \Theta \big)
\big(
K_1^{\textcolor{black}{\ell}}
+
K_2^{\textcolor{black}{\ell}}  
\CM_2^{\textcolor{black}{\ell}}
+
K_2^{\textcolor{black}{\ell}}
\sigma
\big)^{\frac12},
& 
\Theta
:=  &
\tfrac{\log \mathcal{K} + 1 }{\eta  } 
+  
\tfrac{1}{L_1}
,
\\  
C_2^{\textcolor{black}{\ell}}
:=  &
\sqrt{2}  e
\big(
\CM_1 (1)
+
\CM_2^{\textcolor{black}{\ell}}
+
4 \sigma
\big)^{1/2},
&\lambda
:=  &
\tfrac{\eta  }{
 \log \mathcal{K} + 1  + 
\eta /L_1} .
\end{aligned}
\end{equation}
\end{theorem}

\begin{proposition}
\label{prop:num_of_iter-RLMC}
Let  assumptions of Theorem \ref{thm:main_thm_for_RLMC} hold. To achieve a given accuracy tolerance $\epsilon>0$ under $\mathcal{W}_2$-distance, a required number of iterations of the RLMC \eqref{Randomization-eq:RLMC} and RSLMC \eqref{Randomization-eq:RSLMC} is of order $
\widetilde{O}
\big(
\frac{\sqrt{d}}{\epsilon}
\big) $.
\end{proposition}
See  subsection 
\ref{Randomization-sec:app:proof-of-main-result-RLMC} for proofs of Theorem \ref{thm:main_thm_for_RLMC} and Proposition \ref{prop:num_of_iter-RLMC}.
In Table \ref{tab:Comparison}, we compare error bounds and the number of iterations of RLMC \eqref{Randomization-eq:RLMC} and RSLMC \eqref{Randomization-eq:RSLMC} required to achieve the accuracy tolerance $\epsilon$ in
$\mathcal{W}_2$ distance in the literature. Clearly, our error bounds match the best ones in the strongly log-concave case and improve upon the best-known convergence rates in non-convex settings, without requiring any additional smoothness condition other than the gradient Lipschitz condition.

\begin{table}[htbp]
\caption{A comparison of non-asymptotic error bounds in $\CW_2$-distance for Langevin samplers.}
\label{tab:Comparison}
\centering
\begin{tabular}{cccccc}
\toprule
&
Algorithm      
& 
Error  bound
& 
Mixing  time 
&
Log-concavity
&
Additional smoothness
condition
\tnote{1}
\\
\midrule
\cite{cheng2018convergence,dalalyan2017further,durmus2019analysis}
& 
LMC
&
$O (\sqrt{dh})$
&
$\Tilde{O}
(d \epsilon^{-2}
)$
& 
Yes 
&
No
\\
\cite{durmus2019high}
& 
LMC
&
$O (dh)$
& 
$\Tilde{O}
(d \epsilon^{-1}
)$
& 
Yes 
&
Condition \eqref{Randomization-eq:intro-Hessian-Lipschitz} 
\\
\cite{li2022sqrt} 
& 
LMC
&
$O (\sqrt{d}h )$
& 
$\Tilde{O}
(d^{\frac12} \epsilon^{-1}
)$
& 
Yes 
&
Condition \eqref{Randomization-eq:intro-Linear_growth_of_the_3rd-order_derivative} 
\\
\cite{yu2024langevin}
& 
RLMC
&
$O (\sqrt{d}h )$
& 
$\Tilde{O}
(d^{\frac12} \epsilon^{-1}
)$
& 
Yes 
&
No
\\
\cite{majka2020non}
& 
LMC
&
$O (\sqrt[4]{dh} )$
& 
$\Tilde{O}
(d \epsilon^{-4}
)$
& 
No
&
No
\\
\cite{Mou2022Improved}
& 
LMC
&
$O (dh )$
& 
$\Tilde{O}
(d \epsilon^{-1}
)$
& 
No
&
Condition \eqref{Randomization-eq:intro-Hessian-Lipschitz}
\\
\cite{Altschuler2024ShiftedCI,yang2025nonasymptotic}
& 
LMC
&
$O (\sqrt{d}h )$
& 
$\Tilde{O}
(d^{\frac12} \epsilon^{-1}
)$
& 
No
&
Condition \eqref{Randomization-eq:intro-Linear_growth_of_the_3rd-order_derivative}
\\
\cite{Altschuler2024ShiftedCI}
& 
RLMC
&
$O (\sqrt{d}h )$
& 
$\Tilde{O}
(d^{\frac12} \epsilon^{-1}
)$
& 
No
&
Condition
\tnote{2}
\\
This work
& 
RLMC/RSLMC
&
$O (\sqrt{d}h )$
& 
$\Tilde{O}
(d^{\frac12} \epsilon^{-1}
)$
& 
No
&
No
\\
\bottomrule
\end{tabular}
\begin{tablenotes}
\item[1] 
Smoothness assumptions other than the gradient Lipschitz condition for the potential function $U$.
\item[2] 
$U$ is twice-continuously differentiable.
\end{tablenotes}
\end{table}

\subsection{Non-asymptotic error bounds for a modified randomized \textcolor{black}{(splitting)} Langevin Monte Carlo}
\label{Randomization-subsec:Main_result_for_pRLMC}

In this subsection, we intend to show the non-asymptotic error
bound in $\CW_2$-distance for a modified randomized \textcolor{black}{(splitting)} Langevin Monte Carlo in a non-globally Lipschitz setting.
In the sequel, we denote $F(x):=- \nabla U (x)$, $x \in \R^d$ for convenience.
We make the following non-globally Lipschitz condition on $F$.

\begin{assumption}
\label{Randomization-ass:poly_growth_condition}
Let the drift 
$F:=-\nabla U$  of the Langevin dynamics \eqref{Randomization-eq:langevin_SDE} satisfy 
a polynomial growth  condition, i.e., there exists a  dimension-independent constant $L_2>0$ and $\gamma >  0$ such that for any $x,y \in \R^d$, 
\begin{equation}
\label{Randomization-eq:poly_growth_cond}
\big|
F(x) 
-
F(y) 
\big| 
\leq 
L_2 
\big(
1
+
|x|^{\gamma}
+
|y|^\gamma 
\big)
|  x  -  y|.
\end{equation}
\end{assumption}
This immediately implies
\begin{equation}
\label{Randomization-eq:poly_growth_cond1}
|F (x)  | 
\leq 
L_2' d^{\frac12}
+
2 L_2  |  x |^{\gamma+1},
\quad 
\forall \, x \in \R^d,
\end{equation}
where $L_2' d^{\frac12}:=|F (0)| + \gamma  L_2$.
 As shown by \cite{Hutzenthaler2011Strong},
the usual explicit Euler discretization scheme for such SDEs fails to be convergent over finite time. 
To obtain convergent approximations of the Langevin dynamics \eqref{Randomization-eq:langevin_SDE} with super-linear growing nonlinearities, we introduce a projection operator 
\begin{equation} \label{Randomization-eq:defn-projection-oprator}
\cT^h (x):=
\left\{
\begin{aligned}
& \min \{1, 
\vartheta 
d^{\frac{1}{2\gamma+2}}
h^{-\frac{1}{2\gamma+2}}
|x|^{-1}\}x, 
&   &x \neq 0,
\\
& 0,
&  &x=0,
\end{aligned}
\right.
\quad 
\forall \, 
x \in \R^d,
\end{equation}
where $\gamma$ comes from \eqref{Randomization-eq:poly_growth_cond}.
Using this projection operator, we propose the projected randomized \textcolor{black}{(splitting)} Langevin Monte Carlo (pR(S)LMC) algorithms as follows:
\begin{equation}
\label{Randomization-eq:pRLMC}
\begin{aligned}
\Bar{Y}_{n+1}^{\tau,\textcolor{black}{\ell}}
=  &
\Bar{Y}_n^{\textcolor{black}{\ell}} 
+
\textcolor{black}{\ell}
F (\cT^h
(
\Bar{Y}_n^{\textcolor{black}{\ell}} 
)) 
\tau_{n+1} h 
+ 
\sqrt{2} 
\Delta W_{n+1}^\tau,
\quad
\Bar{Y}_0^{\textcolor{black}{\ell}}=x_0,
\\
\Bar{Y}_{n+1}^{\textcolor{black}{\ell}}
=  &
\cT^h
(   \Bar{Y}_n^{\textcolor{black}{\ell}}   ) 
+
F 
(   \cT^h(
\Bar{Y}_{n+1}^{\tau,\textcolor{black}{\ell}}
)  )
h 
+ 
\sqrt{2} 
\Delta W_{n+1},
\quad
n \in \N_0.
\end{aligned}
\end{equation}

In the same way as the previous subsection, we present the following 
uniform moment bounds and the finite-time convergence of the pR(S)LMC \eqref{Randomization-eq:pRLMC},
whose proofs can be found in  subsection \ref{Randomization-sec:app:proofs_of_non_Lip_case}.

\begin{proposition}[Uniformly bounded moments of pR(S)LMC]
\label{prop:uniform_in_time_bounded_moments_to_pRLMC}
Let Assumptions \ref{Randomization-ass:diss_cond_and_grad_Lip_cond}, \ref{Randomization-ass:poly_growth_condition} hold.
Let the uniform stepsize  $h>0$ satisfy $h \leq 
1 \wedge \tfrac{1}{\mu}\wedge \tfrac{1}{d^{\gamma}}$.
Let $\{\Bar{Y}_n^{\textcolor{black}{\ell}} \}_{n\in \N_0}$ be the projected  randomized \textcolor{black}{(splitting)} Langevin Monte Carlo  \eqref{Randomization-eq:pRLMC}. 
Then there exists a dimension-independent constant $\CM_3^{\textcolor{black}{\ell}} $, depending on $\mu,\mu',\vartheta,\gamma,,\textcolor{black}{\ell}, L,L_2,L_2'$, such that, for any $n \in \N_0$,
\begin{equation}
\label{Randomization-eq:uniform_in_time_bounded_moments_to_pRLMC}
\E 
\big[ 
| \Bar{Y}_n^{\textcolor{black}{\ell}}  |^{2p}
\big]
\leq 
e^{-\frac{\mu t_{n}}{2}} 
\E
\big[
| x_0 |^{2p}
\big]  
+
\tfrac{2\CM_3^{\textcolor{black}{\ell}} 
d^{p}}{\mu}.
\end{equation}
\end{proposition}

\begin{proposition}[Finite-time error analysis of pR(S)LMC]
\label{prop:finite_time_error_analysis_pRLMC}
Let  Assumptions \ref{Randomization-ass:diss_cond_and_grad_Lip_cond}, \ref{Randomization-ass:poly_growth_condition} hold. Let $\{X_t\}_{t\geq 0}$ and $\{\Bar{Y}_n^{\textcolor{black}{\ell}} \}_{n\in \N_0
}$ be solutions of the Lanegvin SDE \eqref{Randomization-eq:langevin_SDE} and its randomized \textcolor{black}{(splitting)} approximation \eqref{Randomization-eq:pRLMC}, respectively. If  $h \leq 1 \wedge \tfrac{1}{2L}\wedge \tfrac{1}{\mu}\wedge \tfrac{1}{d^{\gamma}}$, then for a fixed $T=n_1h$, $n_1\in \N$ we have
\begin{equation}
\sup_{n\in [n_1]}
\E
\big[
|
X_{t_n} 
-
\Bar{Y}_n^{\textcolor{black}{\ell}}  
|^2
\big]
\leq 
\exp
\big( (1 + 10L + 6 L_2)  T\big)
\Bar{K}^{\textcolor{black}{\ell}} 
\big(
 d^{(11\gamma+2)/2}
+
 d^{-4}
\E
\big[
|  x_0 |^{11\gamma+10}
\big]
\big)h^2,
\end{equation}
where $\Bar{K}^{\textcolor{black}{\ell}} $ depends on $\mu,\mu',\vartheta,\gamma,\textcolor{black}{\ell} ,L_2,L_2'$, independent of $d$.
\end{proposition}

These estimates together with Propositions \ref{prop:uniform_in_time_bounded_moments_to_LSDE}, \ref{pro:exponential_convergence} help us show the following theorem.

\begin{theorem}[Main result for pR(S)LMC]
\label{thm:main_thm_for_pRLMC}
Let Assumptions \ref{Randomization-ass:diss_cond_and_grad_Lip_cond},\ref{Randomization-ass:one-sided-lip}, \ref{Randomization-ass:log_Sobolev_inequality} and \ref{Randomization-ass:poly_growth_condition} be satisfied.
Let $h$ be the uniform stepsize with  $h \leq 1 \wedge \tfrac{1}{2L}\wedge \tfrac{1}{\mu} \wedge \tfrac{1}{d^{\gamma}}$ and let $\Bar{q}_n^{\textcolor{black}{\ell}} $ denote the transition probability of the randomized \textcolor{black}{(splitting)} Langevin Monte Carlo \eqref{Randomization-eq:pRLMC} at time $t_n:=nh$.
If there exists a  constant $\sigma_p$, only depending on $p$, such that 
$
\label{Randomization-eq:initial_value_condition2}
\E \big[   | x_0 |^{2p} \big]
\leq 
\sigma_p d^p,
$
then  there exist two constants $\Bar{C}_1^{\textcolor{black}{\ell}}$ and $\Bar{C}_2^{\textcolor{black}{\ell}}$, independent of $d$, such that  the law $\nu \Bar{q}_n$ of the $n$-th iterate $\Bar{Y}_n$ of the pR(S)LMC algorithm \eqref{Randomization-eq:pRLMC} satisfies
\begin{equation}
\label{Randomization-eq:main_result_non_lip}
\mathcal{W}_2 (\nu \Bar{q}_n^{\textcolor{black}{\ell}} , \pi )
\leq
\Bar{C}_1^{\textcolor{black}{\ell}}  d^{(11\gamma+2)/4} h  
+
\Bar{C}_2^{\textcolor{black}{\ell}}  \sqrt{d} e^{-\lambda_1 n h}
\end{equation}
for any $n\in \N$ and initial distribution $\nu :=\mathcal{L}(x_0)$, where $\lambda_1:=\frac{\eta}{\log \mathcal{K  } +1 +\eta/(2L)}$.
\end{theorem}

As a direct consequence of Theorem \ref{thm:main_thm_for_pRLMC}, we obtain the mixing time of pR(S)LMC \eqref{Randomization-eq:pRLMC}.
\begin{proposition}
\label{prop:num_of_iter-pRLMC}
Let  assumptions of Theorem \ref{thm:main_thm_for_pRLMC} hold. To achieve a given accuracy tolerance $\epsilon>0$ under $\mathcal{W}_2$-distance, a required number of iterations of the pR(S)LMC \eqref{Randomization-eq:pRLMC} is of order $
\widetilde{O}
\big(
\frac{ d^{(11\gamma+2)/4} }{\epsilon}
\big) $.
\end{proposition}
The proofs of  Theorem \ref{thm:main_thm_for_pRLMC} and Proposition \ref{prop:num_of_iter-pRLMC} are  similar to those of Theorem \ref{thm:main_thm_for_RLMC}  and Proposition \ref{prop:num_of_iter-RLMC}, respectively. We thus omit them here.

\subsection{Technical Overview}
\label{Randomization-subsec:technical_overview}

In this subsection we intend to provide an overview of the present non-asymptotic error analysis of the sampling algorithms.
 
For an approximation $\{\Tilde{Y}_n\}_{n \geq 0}$  to the SDE $\{X_t\}_{t \geq 0}$, the  goal of long-time error analysis is to bound $\mathcal{W}_2 (\nu \Tilde{p}_n, \pi )$,   where $\pi \in \mathcal{P}(\R^d)$ is the invariant distribution of $\{p_t\}_{t \geq 0}$ and $\{\Tilde{p}_n\}_{n \geq 0}$ is the transition  semigroups associated to $\{\Tilde{Y}_n\}_{n \geq 0}$. By the triangle  inequality and for a fixed time $T:=n_1h$, we have
\begin{equation}
\begin{aligned}
\label{Randomization-eq:tri-inequ-for-w2-framework} 
\mathcal{W}_2  (\nu \Tilde{p}_n, \pi )
\leq  
\underbrace{
\mathcal{W}_2  (\nu \Tilde{p}_{n-n_{1}}
\Tilde{p}_{n_{1}},
\nu \Tilde{p}_{n-n_{1}} p_{T} )  
}_{\text{Finite-time error}}
+
\underbrace{
\mathcal{W}_2  ( \nu \Tilde{p}_{n-n_{1}} p_{T}, 
\pi)
}_{\text{Exponential ergodicity}},
\quad
n \geq n_1.
\end{aligned}
\end{equation}
Armed with the triangle inequality, we give an overview of four steps that comprise the proof of Theorem \ref{thm:main_thm_for_RLMC} and \ref{thm:main_thm_for_pRLMC}.

\textbf{Step 1.} 
Uniform-in-time moment estimates are proved for the Langevin SDEs, with the help of dissipativity conditions (see Proposition \ref{prop:uniform_in_time_bounded_moments_to_LSDE}). In addition, the numerical approximations are also shown to have uniform-in-time moment bounds (see Propositions \ref{prop:uniform_in_time_bounded_moments_to_RLMC} and \ref{prop:uniform_in_time_bounded_moments_to_pRLMC}).

\textbf{Step 2.} 
The finite-time mean-square convergence rates are established, suffering from exponential time dependence (see Propositions \ref{prop:finite_time_error_analysis_RLMC} and \ref{prop:finite_time_error_analysis_pRLMC}). These are then used to handle the first term on the right-hand side of \eqref{Randomization-eq:tri-inequ-for-w2-framework}. 
We explicitly show the dependence of error constant on time $T $. 
Accordingly,  one can derive from the definition of the $\mathcal{W}_2$-distance that 
\begin{equation}
\label{Randomization-eq:tri-ineq-part1}
\begin{aligned}
\mathcal{W}_2  (\nu \Tilde{p}_{n-n_{1}}
\Tilde{p}_{n_{1}},
\nu \Tilde{p}_{n-n_{1}} p_{T} ) 
\leq  
C(T)
h.
\end{aligned}
\end{equation}

\textbf{Step 3.}
To estimate  the second  term on the right-hand side of \eqref{Randomization-eq:tri-inequ-for-w2-framework}, we rely on the exponential ergodicity of $\{ p_t \}_{t \geq 0}$ (see Proposition \ref{pro:exponential_convergence}).
In virtue of the monotonicity condition  and LSI, one can achieve the exponential ergodicity as follows:
\begin{equation}
\label{Randomization-eq:tri-ineq-part2}
\mathcal{W}_2  ( \nu \Tilde{p}_{n-n_{1}} p_{T}, 
\pi)
\leq 
\mathcal{K}
e^{-\eta T  }
\mathcal{W}_2  (\nu \Tilde{p}_{n-n_{1}}, \pi ).
\end{equation}
This is the key ingredient to the uniform-in-time error analysis of the sampling algorithms.

\vspace{0.2cm}
\textbf{Step 4.}
The fourth  step is to bound $\mathcal{W}_2  (\nu \Tilde{p}_n, \pi )$. 
Collecting \eqref{Randomization-eq:tri-ineq-part1} and \eqref{Randomization-eq:tri-ineq-part2} together and  choosing   $T=\Theta$ such that  $\mathcal{K}
e^{-\eta T } = \frac{1}{e}$,
one can derive from the uniform-in-time bounded moments (see Theorems \ref{thm:main_thm_for_RLMC} and \ref{thm:main_thm_for_pRLMC}) that 
\begin{equation}
\begin{aligned}
\mathcal{W}_2  (\nu \Tilde{p}_n, \pi )
\leq   
C(\Theta)
h
+
\tfrac{1}{e}
\mathcal{W}_2  (\nu \Tilde{p}_{n-n_{1}}, \pi ).
\end{aligned}
\end{equation}
By iteration,
we have 
\begin{equation}
\mathcal{W}_2  (\nu \Tilde{p}_n, \pi )
\leq 
K_1
h
+
K_2
e^{-\lambda  n h},
\end{equation}
as required.

\section{Proofs of main results for randomized (splitting) Langevin Monte Carlo}
\label{Randomization-sec:app:proofs_of_Lip_case}

To simplify the notation, in the following two sections we will use $Y_n$ in place of $Y_n^\ell$ and $Y_{n+1}^\tau$ in place of $Y_{n+1}^{\tau, \ell}$.

\subsection{Proof of Proposition \ref{prop:uniform_in_time_bounded_moments_to_RLMC}}

\noindent
\textit{Proof of Proposition \ref{prop:uniform_in_time_bounded_moments_to_RLMC}.}
We first recast the \textcolor{black}{unified algorithm \eqref{Randomization-eq:unified-process}}
as
\begin{equation}
Y_{n+1}
=  
Y_n 
-
\nabla U  (Y_{n})  h 
+ 
\sqrt{2} 
\Delta W_{n+1}
-
\big(
\nabla U (Y_{n+1}^\tau)
-
\nabla U (Y_{n})
\big)
h,
\quad
n\in \N_0.
\end{equation}
Taking square on both sides shows 
\begin{equation}
\begin{aligned}
| Y_{n+1} |^2
=   &
|  Y_n |^2
+
h^2 
| \nabla U (Y_n )  |^2
+ 
2 
| \Delta W_{n+1} |^2
+
h^2
|
\nabla U ( Y_{n+1}^\tau)
-
\nabla U (Y_{n})
|^2
\\  & 
-
2h 
\big\langle
Y_{n},
\nabla U (Y_{n})  
\big\rangle
+
2 \sqrt{2} 
\big\langle
Y_{n},
\Delta W_{n+1}
\big\rangle
-
2h 
\big\langle
Y_{n},
\nabla U (Y_{n+1}^\tau)
-
\nabla U (Y_{n})
\big\rangle
\\  & 
-
2 \sqrt{2} h 
\big\langle
\nabla U (Y_{n}) ,
\Delta W_{n+1}
\big\rangle
+
2  h^2
\big\langle
\nabla U (Y_{n}) ,
\nabla U (Y_{n+1}^\tau)
-
\nabla U (Y_{n})
\big\rangle
\\  & 
-
2 \sqrt{2} h 
\big\langle
\Delta W_{n+1},
\nabla U (Y_{n+1}^\tau)
-
\nabla U (Y_{n})
\big\rangle.
\end{aligned}
\end{equation}
Thanks to the Cauchy-Schwarz inequality, \eqref{Randomization-eq:linear_growth_cond} and the dissipativity condition \eqref{Randomization-eq:diss_cond}, one 
can take expectations on both sides to obtain 
\begin{equation}
\label{Randomization-eq:pri_RLMC_Bound_moment}
\begin{aligned}
\E
\big[
| Y_{n+1}  |^2
\big]
=  &
\E
\big[
| Y_{n}  |^2
\big]
+
h^2
\E
\big[
| \nabla U (Y_{n})    |^2
\big]
+
2
\E
\big[
| \Delta W_{n+1} |^2
\big]
+
h^2
\E
\big[
\big|
\nabla U (Y_{n+1}^\tau)
-
\nabla U (Y_{n}) 
\big|^2
\big]
\\   &
-2h
\E
\big[
\big\langle
Y_{n},
\nabla U (Y_{n})  
\big\rangle
\big]
-
2  h
\E
\big[
\big\langle
Y_{n},
\nabla U (Y_{n+1}^\tau)
-
\nabla U (Y_{n})\big\rangle
\big]
\\  &
+
2h^2
\E
\big[
\big\langle
\nabla U (Y_{n}) ,
\nabla U (Y_{n+1}^\tau)
-
\nabla U (Y_{n})
\big]
\big\rangle
\\  &
-
2 \sqrt{2} h 
\E
\big[
\big\langle
\Delta W_{n+1},
\nabla U (Y_{n+1}^\tau)
-
\nabla U (Y_{n})
\big\rangle
\big]
\\ 
\leq   &
\big(
1  -  
( 2\mu - L_1^2 h)
h
\big)
\E
\big[
|  Y_{n}  |^2
\big]
+
2h^2
\E
\big[
|  \nabla U  ( Y_{n})    |^2
\big]
+
4
\E
\big[
| \Delta W_{n+1} |^2
\big]
\\  &
+
\big(
3h^2
+
\tfrac{1}{L_1^2}
\big)
\E
\big[
\big|
\nabla U (Y_{n+1}^\tau)
-
\nabla U (  Y_{n}  )
\big|^2
\big]
+ 
2\mu' d h
\\  
\leq  &
\big(
1  -  
(  2\mu 
   - 
   5 L_1^2 h 
)h
\big)
\E
\big[
|Y_{n}  |^2
\big]
+
\big(
3L_1^2h^2
+
1
\big)
\E
\big[
\big|
Y_{n+1}^\tau
-
Y_{n}
\big|^2
\big]
\\   &
+
4 d h
+
4L_1'd h
+
2\mu' d h,
\end{aligned}
\end{equation}
where the third step holds true as $h \leq 1 \wedge  \tfrac{1}{L_1'}$.
Next, we handle the second item.
Recalling \eqref{Randomization-eq:unified-process} and noting that $|\tau_{n+1}|^2\leq 1$, one can employ \eqref{Randomization-eq:linear_growth_cond} to attain
\begin{equation}
\begin{aligned}
\E
\big[
\big|
Y_{n+1}^\tau
-
Y_{n}
\big|^2
\big]
\leq   &
2 \textcolor{black}{\ell^2}
h^2 
\E
\big[
|\tau_{n+1}|^2
\big|
\nabla U (Y_{n}  )
\big|^2
\big]
+
4
\E
\big[
| \Delta W_{n+1}^\tau |^2
\big]
\\ 
\leq  &
4 \textcolor{black}{\ell^2} L_1^2 h^2
\E
\big[
|  Y_{n}  |^2
\big]
+
4 \textcolor{black}{\ell^2}
L_1'^2 d h^2
+
4 d h.
\end{aligned}
\end{equation}
Inserting this into \eqref{Randomization-eq:pri_RLMC_Bound_moment}, together with $h \leq 1 \wedge  \tfrac{1}{L_1}
\wedge 
\tfrac{1}{L_1'}
\wedge
\tfrac{\mu}{\textcolor{black}{(5 + 16 \ell^2)}L_1^2}$, yields 
\begin{equation}
\begin{aligned}
    \E
\big[
| Y_{n+1}  |^2
\big]
\leq   &
\big(
1  -  
(  2\mu 
   - 
   5 L_1^2 h 
)h
\big)
\E
\big[
| Y_{n}  |^2
\big]
+
\big(
3L_1^2h^2
+
1
\big)
\big(
4  \textcolor{black}{\ell^2}
L_1^2 h^2
\E
\big[
|Y_{n}   |^2
\big]
+
4 \textcolor{black}{\ell^2} L_1'^2 d h^2
+
4 d h
\big)
\\   &
+
4 d h
+
4L_1'd h
+
2\mu' d 
\\ 
\leq   &
\big(
1  -  
(  2\mu 
   - 
   \textcolor{black}{(5 + 16 \ell^2)} 
   L_1^2 h 
)h
\big)
\E
\big[
| Y_{n}  |^2
\big]
+
\big(
20 + \textcolor{black}{(4 +  16 \ell^2)} L_1' + 2 \mu'
\big) dh 
\\
\leq    &
\big(
1 - \mu h
\big)
\E
\big[
| Y_{n}  |^2
\big]
+
\big(
20 + \textcolor{black}{(4 +  16 \ell^2)} L_1' + 2 \mu'
\big) dh .
\end{aligned}
\end{equation}
By iteration and using the inequality $1-u \leq e^{-u}, u >0$ and $h \leq  \frac{1
}{\mu}$ show 
\begin{equation}
\begin{aligned}
    \E
\big[
| Y_{n+1}  |^2
\big]
\leq   &
\big(
1 - \mu h
\big)^{n+1}
    \E
\big[
| x_0|^2
\big]
+
\big(
20 + \textcolor{black}{(4 +  16 \ell^2)} L_1'^2h + 2 \mu'
\big) dh 
\sum_{i=1
}^n
\big(
1 - \mu h
\big)^{i}
\\  
\leq   & 
e^{-\mu t_{n+1}}
    \E
\big[
| x_0|^2
\big]
+
\tfrac{(
20 + \textcolor{black}{(4 +  16 \ell^2)} L_1'^2h + 2 \mu'
) d}{\mu},
\end{aligned}
\end{equation}
as required.
\qed

\subsection{Proof of Proposition \ref{prop:finite_time_error_analysis_RLMC}}

Before proving  Proposition \ref{prop:finite_time_error_analysis_RLMC}, we first introduce some useful lemmas.

\begin{lemma}
\label{lem:holder_continuous_of_LSDE_lip}
Let 
Assumptions  \ref{Randomization-ass:diss_cond_and_grad_Lip_cond} and \ref{Randomization-ass:Lip_condition}  be  fulfilled. Let $X(s,x;t)$  denote the solution to the  Langevin SDE \eqref{Randomization-eq:langevin_SDE} at $t$, starting from the initial value $x$ at $s$. 
If the uniform stepsize $h>0$ satisfies $h \leq 
1 \wedge \tfrac{1}{L_1}\wedge\tfrac{1}{L_1'}$,
then for any $x\in \R^d$, any $0 <\theta \leq h$ and $0\leq s  \leq t$, it holds that
\begin{equation}
\label{Randomization-eq:Holder_continuous_of_X_t_lip}
\begin{aligned}
& \E_W
\big[
| 
X(s,x;t+\theta ) 
-  
X(s,x;t )   
|^{2p}
\big]
\\ 
\leq   &
\Big(
\big(
2^{4p-2}
L_1'^{p}
+
2^{4p-2}
\CM_1(p)
L_1^p
+
2^{3p-1} (2p-1)!!
\big)
d^{p}
+
2^{4p-2} 
L_1^{p}
|
x
|^{2p}
\Big)
\theta^{p}.
\end{aligned}
\end{equation}
Moreover, for any  any $x\in \R^d$, any $0 <\theta \leq h$ and $0\leq s  \leq t$, it holds that
\begin{equation}
\label{Randomization-eq:Holder_continuous_of_F_X_t_lip}
\begin{aligned}
\E_W
\big[
\big|
\nabla U  (X(s,x;t+\theta ) )
-
\nabla U (X(s,x;t) )
\big|^2
\big] 
\leq   &
\Big(
\big(
4
L_1^2 
L_1'
+
4
\CM_1(1)
L_1^3 
+
4L_1^2 
\big)
d
+
4
L_1^3
|
x
|^{2}
\Big)
\theta.
\end{aligned}
\end{equation}
\end{lemma}
\begin{proof}
First, according to the Langevin SDE \eqref{Randomization-eq:solutions_of_LSDE}, we have, for any $x\in \R^d$, any $0 <\theta \leq h$ and $0\leq s  \leq t$,
\begin{equation}
X(s,x;t+\theta )  
-  
X(s,x;t ) 
=
\int_t^{t+\theta} 
-\nabla U (X_r) \,
\dd r 
+
\int_s^{t+\theta}
\sqrt{2}\,
\dd W_r.
\end{equation}
By taking $2p$-th power on both sides and then expectation $\E_W$, one can use the assumption $h \leq 
1 \wedge \tfrac{1}{L_1}\wedge\tfrac{1}{L_1'}$, the inequality 
\begin{equation}
\label{Randomization-eq:fun_ineq}
\Big(
\sum_{i=1}^k u_i 
\Big)^q 
\leq 
k^{q-1} 
\sum_{i=1}^k u_i^q,\quad q\geq 1,\quad u_i\in \R,
\end{equation} the H\"older inequality and \eqref{Randomization-eq:poly_growth_cond1} to attain, for any $p\geq 1$,
\begin{equation}
\label{Randomization-eq:estimate_holder_continuous_X_t}
\begin{aligned}
& \E_W
\big[
|
X(s,x;t+\theta )  
-  
X(s,x;t )
|^{2p}
\big]
\\ 
\leq   &
2^{2p-1}
\E_W
\bigg[
\bigg|
\int_t^{t+\theta} 
\nabla U (X(s,x;t ))\,
\dd r
\bigg|^{2p}
+
\bigg|
\int_t^{t+\theta}
\sqrt{2}
\,
\dd W_r
\bigg|^{2p}
\bigg]
\\ 
\leq  &
2^{2p-1}
\theta^{p} 
\bigg(
\theta^{p-1}
\int_t^{t+\theta}
\E_W
\big[
\big|
\nabla U (X(s,x;t ))\,
\big|^{2p}
\big]\,
\dd r
+
2^p (2p-1)!! d^p  
\bigg) 
\\ 
\leq  & 
2^{2p-1} \theta^{p}
\bigg(
2^{2p-1}
\theta^{p-1}
\int_t^{t+\theta}
\Big(
L_1^{2p}
\E_W
\big[
\big|
X(s,x;t )
\big|^{2p(\gamma + 1)}
\big]
+
L_1'^{2p}
d^p 
\Big)
\,
\dd r
+
2^p (2p-1)!! d^p 
\bigg)
\\ 
\leq  & 
2^{2p-1} \theta^{p}
\bigg(
2^{2p-1}
\theta^{p-1}
\int_t^{t+\theta}
\Big(
L_1^{2p}
|
x
|^{2p}
+
L_1'^{2p}
d^p 
+
\CM_1 (p) L_1^{2p} d^{p}
\Big)
\,
\dd r
+
2^p (2p-1)!! d^p  
\bigg)
\\ 
\leq  &
\Big(
\big(
2^{4p-2}
L_1'^{p}
+
2^{4p-2}
\CM_1(p) 
L_1^p
+
2^{3p-1} (2p-1)!!
\big)
d^{p}
+
2^{4p-2} 
L_1^{p} 
|
x
|^{2p}
\Big)
\theta^{p},
\end{aligned}
\end{equation}
where the fourth step holds true due to Proposition \ref{prop:uniform_in_time_bounded_moments_to_LSDE}.
The first assertion \eqref{Randomization-eq:Holder_continuous_of_X_t_lip} is thus proved.
Now, we validate \eqref{Randomization-eq:Holder_continuous_of_F_X_t_lip}.
Again,
thanks to the assumption $h \leq 1 \wedge \tfrac{1}{L_1} \wedge \tfrac{1}{L_1'} $ and Proposition \ref{prop:uniform_in_time_bounded_moments_to_LSDE},
using \eqref{Randomization-eq:Holder_continuous_of_X_t_lip},  the H\"older inequality  and the linear growth condition \eqref{Randomization-eq:linear_growth_cond} gives, for any $x\in \R^d$, any $0 <\theta \leq h$ and $0\leq s  \leq t$,
\begin{equation}
\begin{aligned}
& \E_W
\big[
\big|
\nabla U  (X(s,x;t+\theta ) )
-
\nabla U (X(s,x;t) )
\big|^2
\big]
\\ 
\leq   &
L_1^2 \,
\E_W
\Big[
\big|
X(s,x;t+\theta )
-
X(s,x;t )
\big|^2
\Big]
\\ 
\leq   &
\Big(
\big(
4
L_1^2 
L_1'
+
4
\CM_1(1)
L_1^3
+
4L_1^2 
\big)
d
+
4
L_1^3
|
x
|^{2}
\Big)
\theta.
\end{aligned}
\end{equation}
We thus complete this proof.
\end{proof}

%
For the finite-time error analysis for RLMC \eqref{Randomization-eq:RLMC} and RSLMC \eqref{Randomization-eq:RSLMC}, we introduce a \textcolor{black}{unified} one-step approximation, given by
\begin{equation}
\label{Randomization-eq:one_step_RLMC}
\begin{aligned}
&  Y_m(t,x;t+ \tau h) 
:=  
x 
-
\textcolor{black}{\ell}
\nabla U (  x  ) \tau h
+
\sqrt{2}   
( W_{t+\tau h}  -  W_t),
\quad 
\textcolor{black}{\ell \in \{0,1\},\quad 
\tau \sim \mathcal{U}(0,1)},
\\ 
&
Y(t,x;t+h) 
:=  
x 
-
\nabla U 
(  Y_m(t,x;t+ \tau h)  )  h
+
\sqrt{2}   
( W_{t+ h}  -  W_t),
\quad 
\textcolor{black}{
\forall \ 
t \geq 0, \;
x\in \R^d.}
\end{aligned}
\end{equation}
\textcolor{black}{Evidently, $Y(t,x;t+h)$ coincides with the one-step update of RLMC \eqref{Randomization-eq:RLMC} for $\ell = 1$, and corresponds to the one-step update of RSLMC \eqref{Randomization-eq:RSLMC} for $\ell = 0$.}
Equipped with Lemma \ref{lem:holder_continuous_of_LSDE_lip} and Assumption \ref{Randomization-ass:Lip_condition}, one can establish the following one-step error bounds.
\begin{lemma}
\label{lem:onestep_waek_and_strong_error_of_RLMC}
Let Assumption \ref{Randomization-ass:Lip_condition} be satisfied. 
Let $X(t,x;t+h)$ denote the solution to the  Langevin SDE \eqref{Randomization-eq:langevin_SDE} at $t+h$, starting from the initial value $x $ at $t$. If the uniform stepsize $h$ satisfies $h \leq 
1 \wedge \tfrac{1}{L_1}\wedge\tfrac{1}{L_1'}$, then, for all $x \in \R^d$ and $t\in[0, +\infty)$, the one-step approximation \eqref{Randomization-eq:one_step_RLMC} obeys
\begin{equation}
\label{Randomization-eq:one-step-str-and-weak-error}
\begin{aligned}
  \big|
\E 
\big[
X(t,x;t+h)
-
Y(t,x;t+h)
\big]
\big| 
\leq    &
\textcolor{black}{
\Big(
\big(
(4 +  4 \ell^2)
L_1^2
L_1'^2
+
4
\CM_1(1)
L_1^4
\big)
d
+
(4 +  4 \ell^2)
L_1^4
|
x
|^{2}
\Big)^{\frac12}
h^{2}},
\\  
  \Big(
\E 
\big[
\big|
X(t,x;t+h)
-
Y(t,x;t+h)
\big|^2
\big]
\Big)^{\frac12}  
\leq    &
\textcolor{black}{
\Big(
(16  +  8\ell^2)
L_1^2 
L_1'
+
16
\CM_1(1)
L_1^3 
+
8L_1^2 
\big)
d
+
(16  +  8\ell^2)
L_1^3
|
x
|^2 
\Big)^{\frac12}
h^{\frac32}.}
\end{aligned}
\end{equation}
\end{lemma}
\vspace{-1em}
\begin{proof}
Recalling \eqref{Randomization-eq:solutions_of_LSDE} and \eqref{Randomization-eq:one_step_RLMC}, we have, for all $t\in[0,+\infty)$, $\tau \sim \mathcal{U}(0,1)$ and $h \in (0,1)$, 
\begin{equation}
\label{Randomization-eq:difference-of-one-step-lsde-and-RLMC}
X(t,x;t+h)
-
Y(t,x;t+h)
=
h \nabla U 
(  Y_m(t,x;t+ \tau h)  )
-
\int_{t}^{t+h}
\nabla U    (X(t,x;s)) 
\dd s .
\end{equation}
We prove the first assertion of \eqref{Randomization-eq:one-step-str-and-weak-error} first.
Since, for any $Z \in L^p ([0,T]\times \Omega_W; \R^d)$ and $t \in [0,+\infty)$,  $h\in (0,1)$,
\begin{equation}
\label{Randomization-eq:property_of_uniform-distribution}
\int_{t}^{t+h}
Z(s,\omega)
\, \dd s 
=
h \int_{0}^{1}
Z(t  + s h, \omega) 
\, \dd s 
=
h
\E_{\xi}
\big[
Z(t  + \xi h, \omega)
\big],
\quad 
\forall \, 
\xi \sim \mathcal{U}(0,1),
\end{equation}
one can show
\begin{equation}
\int_{t}^{t+h}
\nabla U    (X(t,x;s))
\dd s 
= 
h \E_{\tau} \big[
\nabla U    (X(t,x;t+ \tau h))
\big], \quad  \forall \,
\tau \sim \mathcal{U}(0,1).
\end{equation}
Bearing this in mind one can derive from \eqref{Randomization-eq:difference-of-one-step-lsde-and-RLMC} that
\begin{equation}
\label{Randomization-eq:pri-esti-one-step-weak-error}
\begin{aligned}
\big|
\E 
\big[
X(t,x;t+h)
-
Y(t,x;t+h)
\big]
\big|
=  &
\Big|
\E 
\Big[
h 
\nabla U 
(  Y_m(t,x;t+ \tau h)  )
-
\int_{t}^{t+h}
\nabla U    (X(t,x;s))
\dd s
\Big]
\Big|
\\   
=  &
h
\Big|
\E
\Big[
\nabla U 
(  Y_m(t,x;t+ \tau h)  )
-
\E_{\tau} \big[
\nabla U    (X(t,x;t+ \tau h))
\big]
\Big]
\Big|
\\   
=  &
h
\Big|
\E_W
\Big[
\E_{\tau} \big[
\nabla U 
(  Y_m(t,x;t+ \tau h)  )
-
\nabla U    (X(t,x;t+ \tau h))
\big]
\Big]
\Big|
\\   
\leq   &
h
\E_W
\Big[
\E_{\tau} \big[
\big|
\nabla U 
(  Y_m(t,x;t+ \tau h)  )
-
\nabla U    (X(t,x;t+ \tau h))
\big|
\big]
\Big].
\end{aligned}
\end{equation}
By the gradient Lipschitz condition \eqref{Randomization-eq:Lip_cond}, the H\"older inequality and Proposition \ref{prop:uniform_in_time_bounded_moments_to_LSDE}, 
we deduce from \eqref{Randomization-eq:solutions_of_LSDE}  and \eqref{Randomization-eq:one_step_RLMC} that, for $h \leq 
1 \wedge \tfrac{1}{L_1}\wedge\tfrac{1}{L_1'}$
\begin{equation}
\begin{aligned}
&
\E_W
\Big[
\E_{\tau} \big[
\big|
\nabla U 
(  Y_m(t,x;t+ \tau h)  )
-
\nabla U    (X(t,x;t + \tau h))
\big|
\big]
\Big]
\\  
\leq   &
L_1
\E_W
\Big[
\E_{\tau} \big[
\big|
Y_m(t,x;t+ \tau h) 
-
X(t,x;t + \tau h)
\big|
\big]
\Big]
\\ 
=   &
L_1
\E_W
\bigg[
\E_{\tau} \Big[
\Big|
\int_{t}^{t+ \tau h}
\big(
\nabla U    (X(t,x;s)) 
-  
\textcolor{black}{\ell}
\nabla U    (x)
\big)
\dd s 
\Big|
\Big]
\bigg]
\\  
\leq   &
\textcolor{black}{
L_1
\E_{\tau} \Big[
\int_{t}^{t+ \tau h}
\Big(
2 
\E_W
\big[
\big|
\nabla U    (X(t,x;s)) 
\big|^2
\big]
+ 
2
\ell^2
\big|
\nabla U    (x)
\big|^2 
\Big)^{\frac12}
\dd s 
\Big]}
\\  
\leq   &
\textcolor{black}{
L_1
\E_{\tau} \Big[
\int_{t}^{t+ \tau h}
\Big(
4 L_1^2
\E_W
\big[
\big|
X(t,x;s) 
\big|^2 
\big]
+
4 \ell^2 
L_1^2  |x|^2  
+
(4 +  4 \ell^2)
L_1'^2 d
\Big)^{\frac12}
\dd s 
\Big]}
\\  
\leq   &
\textcolor{black}{
\Big(
\big(
(4 +  4 \ell^2)
L_1^2
L_1'^2
+
4
\CM_1(1)
L_1^4
\big)
d
+
(4 +  4 \ell^2)
L_1^4
|
x
|^{2}
\Big)^{\frac12}
h.}
\end{aligned}
\end{equation}
Inserting this into \eqref{Randomization-eq:pri-esti-one-step-weak-error} gives 
\begin{equation}
\big|
\E 
\big[
X(t,x;t+h)
-
Y(t,x;t+h)
\big]
\big|
\leq 
\textcolor{black}{
\Big(
\big(
(4 +  4 \ell^2)
L_1^2
L_1'^2
+
4
\CM_1(1)
L_1^4
\big)
d
+
(4 +  4 \ell^2)
L_1^4
|
x
|^{2}
\Big)^{\frac12}
h^{2}}.
\end{equation}
The first assertion of \eqref{Randomization-eq:one-step-str-and-weak-error} is thus validated. Concerning the other assertion, we recall \eqref{Randomization-eq:difference-of-one-step-lsde-and-RLMC} and employ the H\"older inequality as well as the triangle inequality to get 
\begin{equation}
\label{Randomization-eq:pri-esti-one-step-str-error}
\begin{aligned}
& \E 
\big[
\big|
X(t,x;t+h)
-
Y(t,x;t+h)
\big|^2
\big]
\\
\leq  &
\E 
\Big[
\Big|
\int_{t}^{t+h}
\big(
\nabla U    (X(t,x;s)) 
-
\nabla U  
(  Y_m(t,x;t+ \tau h)  )
\big)
\dd s
\Big|^2
\Big]
\\ 
\leq  &
h 
\int_{t}^{t+h}
\E 
\big[
\big|
\nabla U    (X(t,x;s)) 
-
\nabla U  
(  Y_m(t,x;t+ \tau h)  )
\big|^2
\big]
\dd s
\\ 
\leq  &
\underbrace{
2 h 
\int_{t}^{t+h}
\E 
\big[
\big|
\nabla U    (X(t,x;s)) 
-
\nabla U  
(  X(t,x;t+ \tau h)  )
\big|^2
\big]
\dd s
}_{=:J_1}
\\   &   +
\underbrace{
2 h 
\int_{t}^{t+h}
\E 
\big[
\big|
\nabla U    
(X(t,x;t+ \tau h)) 
-
\nabla U  
(  Y_m(t,x;t+ \tau h)  )
\big|^2
\big]
\dd s
}_{=:J_2}.
\end{aligned}
\end{equation}
In the following we cope with the above two items separately.
By Lemma \ref{lem:holder_continuous_of_LSDE_lip}, we have 
\begin{equation}
J_1 
\leq  
\Big(
\big(
8
L_1^2 
L_1'
+
8
\CM_1(1)
L_1^3 
+
8L_1^2 
\big)
d
+
8
L_1^3
|
x
|^2
\Big)
h^3.
\end{equation}
In virtue of 
the gradient Lipschitz condition and Proposition \ref{prop:uniform_in_time_bounded_moments_to_LSDE},
 one can derive from \eqref{Randomization-eq:solutions_of_LSDE}  and \eqref{Randomization-eq:one_step_RLMC} that 
\begin{equation}
\begin{aligned}
J_2 \leq  &
2 h 
\int_{t}^{t+h}
\E 
\big[
\big|
X(t,x;t+ \tau h)
-
Y_m(t,x;t+ \tau h)  
\big|^2
\big]
\dd s
\\  
\leq  &
2 L_1^2 h 
\int_{t}^{t+h}
\E 
\big[
\big|
X(t,x;t+ \tau h)
-
Y_m(t,x;t+ \tau h)  
\big|^2
\big]
\dd s
\\ 
\leq  &
2 L_1^2 h 
\int_{t}^{t+h}
\E
\Big[
\Big|
\int_{t}^{t+ \tau h}
\nabla U    (X(t,x;r)) 
- 
\textcolor{black}{\ell}
\nabla U    (x)
\dd r 
\Big|^2 
\Big]
\dd s
\\ 
\leq  &
\textcolor{black}{
4 L_1^2 h^2 
\int_{t}^{t+h}
\E_\tau 
\Big[
\int_{t}^{t+ \tau h}
\Big(
\E_W
\big[
\big|
\nabla U    (X(t,x;r))
\big|^2 
\big]
+
\ell^2
\big|
\nabla U    (x)
\big|^2
\Big)
\dd r 
\Big]
\dd s}
\\ 
\leq  &
\textcolor{black}{
4 L_1^2 h^2 
\int_{t}^{t+h}
\E_\tau 
\Big[
\int_{t}^{t+ \tau h}
\Big(
2 
L_1^2
\E_W
\big[
\big|
X(t,x;r)
\big|^2 
\big]
+
2
\ell^2 
L_1^2
|  x  |^2
+
(2+2\ell^2)
L_1'^2
d
\Big)
\dd r 
\Big]
\dd s}
\\ 
\leq  &
\textcolor{black}{
\Big(
\big(
(8  +  8\ell^2)
L_1^2
L_1'^2
+
8
\CM_1(1)
L_1^4
\big)
d
+
(8  +  8\ell^2)
L_1^4
|
x
|^{2}
\Big)
h^4.}
\end{aligned}
\end{equation}
Thanks to 
$h \leq 
1 \wedge \tfrac{1}{L_1}\wedge\tfrac{1}{L_1'}$, plugging estimates of $J_1$ and  $J_2$ into \eqref{Randomization-eq:pri-esti-one-step-str-error} shows 
\begin{equation}
\E 
\big[
\big|
X(t,x;t+h)
-
Y(t,x;t+h)
\big|^2
\big]
\leq 
\textcolor{black}{
\big(
(16  +  8\ell^2)
L_1^2 
L_1'
+
16
\CM_1(1)
L_1^3 
+
8L_1^2 
\big)
d
+
(16  +  8\ell^2)
L_1^3
|
x
|^2
\big) h^3.}
\end{equation}
Now the second assertion in \eqref{Randomization-eq:one-step-str-and-weak-error} is proved.
\end{proof}
Now we are ready to prove Proposition \ref{prop:finite_time_error_analysis_RLMC}.

\vspace{1em}
\noindent
\textit{Proof of Proposition \ref{prop:finite_time_error_analysis_RLMC}.} 
In light of \cite[Theorem 3.3]{yang2025nonasymptotic}, one can combine   Assumptions 
\ref{Randomization-ass:diss_cond_and_grad_Lip_cond},
\ref{Randomization-ass:Lip_condition} with Proposition \ref{prop:uniform_in_time_bounded_moments_to_RLMC}, Lemma \ref{lem:onestep_waek_and_strong_error_of_RLMC} to obtain
\begin{equation}
\E
\big[
|
X_{t_n} 
-
Y_n 
|^2
\big]
\leq 
\exp
\big( 1 +  12 L_1 T\big)
\big(
K_1^{\textcolor{black}{\ell}}  d
+
K_2^{\textcolor{black}{\ell}} 
\E
\big[
|
x_0
|^{2}
\big]
\big)  
h^2,
\end{equation}
where 
\begin{equation}
\begin{aligned}
K_1^{\ell}
:=  &
\textcolor{black}{
(8 0 + 40 \ell^2 )
L_1^2 
L_1'
+  
(4   +   4 \ell^2 ) 
L_1^2 
L_1'^2
+
80  \CM_1(1) L_1^3 
+
4  \CM_1(1) L_1^4  }
\\   & 
\textcolor{black}{
+ 
(8 0 + 40 \ell^2 )
\CM_2^{\ell} L_1^3  
+  
(4   +   4 \ell^2 )
\CM_2^{\ell} L_1^4
+
4 0  L_1^2 
\big),}
\\
K_2^{\ell}
:=  &
\textcolor{black}{
(8 0 + 40 \ell^2) 
L_1^3
+  
(4 + 4   \ell^2)  L_1^4
.}
\end{aligned}
\end{equation}
Thus, we derive the desired assertion.

\subsection{Proofs of Theorem \ref{thm:main_thm_for_RLMC} and Proposition \ref{prop:num_of_iter-RLMC}}
\label{Randomization-sec:app:proof-of-main-result-RLMC}

\textcolor{black}{Similarly as before, we shall henceforth write $q_n$ instead of $q_n^{\ell}$, to simplify the notation.}

\vspace{1em}
\noindent
\textit{Proof of Theorem \ref{thm:main_thm_for_RLMC}.}
By employing the triangle inequality, we obtain that for any $n \geq   n_1$, 
\begin{eqnarray}
\mathcal{W}_2 \big( \nu q_n , \pi \big) 
\leq 
\mathcal{W}_2  \big(
      \nu q_{n-n_1} q_{n_1} , 
     \nu q_{n-n_1} p_{ n_1h}
     \big)
+
\mathcal{W}_2  \big( 
      \nu q_{n-n_1} p_{ n_1h } , 
       \pi
     \big) .
\end{eqnarray}
Now, we   estimate   $\mathcal{W}_2(\nu q_{n-n_1} q_{n_1} , 
\nu q_{n-n_1} p_{ n_1h})$ and $\mathcal{W}_2 ( \nu q_{n-n_1} p_{ n_1h} , 
\pi)$, separately.
Note that 
\begin{eqnarray}
\mathcal{W}_2  \big(
      \nu q_{n-n_1} q_{n_1} , 
     \nu q_{n-n_1} p_{ n_1h}
     \big)
=
\mathcal{W}_2 \big( 
       \cL (
            Y(t_{n-n_1}, Y_{  n-n_1}  ;t_n) 
            ) , 
       \cL (
            X(t_{n-n_1}, Y_{n-n_1}  ;t_n) 
            )
     \big).
\end{eqnarray}
In view of Propositions \ref{prop:uniform_in_time_bounded_moments_to_RLMC},\ref{prop:finite_time_error_analysis_RLMC} and Assumption \ref{Randomization-ass:Lip_condition}, we obtain  
\begin{equation}
\begin{aligned}
   &  \mathcal{W}_2^2 \big( 
       \cL (
            Y(t_{n-n_1}, Y_{  n-n_1}  ;t_n) 
            ) , 
       \cL (
            X(t_{n-n_1}, Y_{n-n_1}  ;t_n) 
            )
     \big)
\\
\leq   &
\E
     \big[
       \big|  
         X(t_{n-n_1}, Y_{  n-n_1}  ;t_n) 
         -  
         Y (t_{n-n_1}, Y_{  n-n_1}  ;t_n)  
       \big|^2
       \big]
\\
\leq   &
\exp
\big( 1  +  12 L_1 T\big)
\big(
K_1^{\textcolor{black}{\ell}}  d
+
K_2^{\textcolor{black}{\ell}}
\E
\big[
|
Y_{  n-n_1}
|^{2}
\big]
\big)  
h^2
\\
\leq   &
\exp
\big( 1  +  12 L_1 T\big)
\big(
K_1^{\textcolor{black}{\ell}}  d
+
K_2^{\textcolor{black}{\ell}}
\E
\big[
|
Y_{  n-n_1}
|^{2}
\big]
\big)  
h^2
\\  
\leq  &
\exp
\big( 1 +  12L_1 T\big)
\Big(
(
K_1^{\textcolor{black}{\ell}}
+
K_2^{\textcolor{black}{\ell}}  
\CM_2^{\textcolor{black}{\ell}} )
d
+
K_2  
\E
[   |  x_0     |^{2}     ] 
\Big) 
h^{2}.
\end{aligned}
\end{equation}
This implies 
\begin{equation}
\begin{aligned}
&  \mathcal{W}_2 \big( 
       \cL (
            Y(t_{n-n_1}, Y_{  n-n_1}  ;t_n) 
            ) , 
       \cL (
            X(t_{n-n_1}, Y_{n-n_1}  ;t_n) 
            )
     \big)
\\
\leq  & 
\exp
\big(1  +12  L_1 T\big)
\Big(
(
K_1^{\textcolor{black}{\ell}}
+
K_2^{\textcolor{black}{\ell}} 
\CM_2^{\textcolor{black}{\ell}} )
d
+
K_2^{\textcolor{black}{\ell}}  
\E
[   |  x_0     |^{2}     ] 
\Big)^{\frac12} 
h.
\end{aligned}
\end{equation}
On the other hand, by applying Proposition \ref{pro:exponential_convergence}, we derive
\begin{eqnarray}
\mathcal{W}_2  
\big( 
      \nu q_{n-n_1} p_{ n_1h } , 
       \pi
     \big) 
\leq
\mathcal{K} e^{-\eta  n_1 h } 
\mathcal{W}_2  \big( 
      \nu q_{n-n_1}  , 
       \pi
     \big) .
\end{eqnarray}
In what follows, for a given timestep $h>0$, 
we select 
\begin{eqnarray}
\label{Randomization-eq:choice_of_M}
n_1
= 
\big\lceil 
      \tfrac{\log \mathcal{K} + 1 }{\eta h} 
\big\rceil,
\end{eqnarray}
for which $n_1$ is an integer.
In view of $h \leq  \frac{1}{L_1}$, we have 
\begin{eqnarray}
T:=n_1h  
\leq
\big( 
    \tfrac{\log \mathcal{K} + 1 }{\eta h } +1
\big)
h
\leq 
\tfrac{\log \mathcal{K} + 1 }{\eta  } 
+  
\tfrac{1}{L_1} =: \Theta.
\end{eqnarray}
Noticing  that 
\begin{equation}
0 <
\mathcal{K} e^{-\eta  n_1 h }   \leq 
e^{-1}
<1,
\end{equation}one can collect the above estimate 
and utilize Lemma D.1 of \cite{yang2025nonasymptotic} to obtain 
\begin{equation}
\begin{aligned}
&
\mathcal{W}_2 \big( \nu q_n , \pi \big) 
\\ 
\leq   &
\exp
\big( 1  +  12 L_1 \Theta \big)
\Big(
(
K_1^{\textcolor{black}{\ell}}
+
K_2^{\textcolor{black}{\ell}}  
\CM_2^{\textcolor{black}{\ell}} )
d
+
K_2^{\textcolor{black}{\ell}}  
\E
[   |  x_0 '     |^{2}     ] 
\Big)^{\frac12} 
h
+
\tfrac{1}{e}
\mathcal{W}_2  \big( 
      \nu q_{n-n_1}  , 
       \pi
     \big) 
\\
\leq   &
\exp
\big( 1  +  12 L_1 \Theta \big)
\Big(
(
K_1^{\textcolor{black}{\ell}}
+
K_2^{\textcolor{black}{\ell}}  
\CM_2^{\textcolor{black}{\ell}} )
d
+
K_2  
\E
[   |  x_0 '     |^{2}     ] 
\Big)^{\frac12} 
h
+
e^{1-\frac{n}{n_1}}
\sup_{k\in [n_1-1]_0} 
\mathcal{W}_2 \big( 
      \nu q_{k}  , 
       \pi
    \big).
\end{aligned}
\end{equation}
Utilizing the definition of $\mathcal{W}_2$-distance and Lemmas \ref{prop:uniform_in_time_bounded_moments_to_LSDE}, \ref{prop:uniform_in_time_bounded_moments_to_RLMC}  leads to 
\begin{equation}
\begin{aligned}
\sup_{k\in [n_1-1]_0} 
\mathcal{W}_2 \big( 
      \nu q_{k}  , 
       \pi
    \big) 
\leq     &
\sup_{k \geq  0} 
\Big(
2 \E
  \big[   |   Y_k   |^2  \big]
+
2  \E
   \big[  |   X_{t_k}  |^2  \big]
\Big)^{\frac12}
\leq    
\Big(
2  
\big(
\CM_1 (1)
+
\CM_2^{\textcolor{black}{\ell}}
\big)
d 
+
4 \E
  [   |  x_0   |^2  ]
\Big)^{\frac12}.
\end{aligned}
\end{equation}
Owing to \eqref{Randomization-eq:choice_of_M}, we get 
\begin{eqnarray}
\tfrac{n}{n_1}  
\geq
\tfrac{n}{
\tfrac{\log \mathcal{K} + 1 }{\eta h  } +1
}
\geq
\tfrac{\eta  n  h}{
 \log \mathcal{K} + 1  + 
\eta /L_1}
=:\lambda n h .
\end{eqnarray}
Thanks to the fact $e^{ -\frac{n}{n_1}}  \leq e^{-\lambda n h}$,
we derive from \eqref{Randomization-eq:initial_value_condition1} that
\begin{equation}
\begin{aligned}
\mathcal{W}_2 
\big( \nu q_n , \pi
\big) 
\leq    &
\exp
\big( 1  +  12 L_1 \Theta \big)
\Big(
(
K_1^{\textcolor{black}{\ell}}
+
K_2^{\textcolor{black}{\ell}}  
\CM_2^{\textcolor{black}{\ell}} )
d
+
K_2^{\textcolor{black}{\ell}}  
\E
[   |  x_0     |^{2}     ] 
\Big)^{\frac12} 
h
\\  &
+
\Big(
2   e^2
\big(
\CM_1 (1)
+
\CM_2^{\textcolor{black}{\ell}}
\big)
d 
+
4 \E
  [   |  x_0   |^2  ]
\Big)^{\frac12}
e^{-\lambda n h}
\\  
\leq  &
\exp
\big( 21 L_1 \Theta \big)
\big(
K_1^{\textcolor{black}{\ell}}
+
K_2^{\textcolor{black}{\ell}}  
\CM_2^{\textcolor{black}{\ell}}
+
K_2^{\textcolor{black}{\ell}} \sigma
\big)^{\frac12}
\sqrt{d } 
h
\\  &
+
\sqrt{2}  e
\big(
\CM_1 (1)
+
\CM_2^{\textcolor{black}{\ell}}
+
4 \sigma
\big)^{\frac12}
\sqrt{d } 
e^{-\lambda n h},
\end{aligned}
\end{equation}
as required.
\qed

\vspace{0.1in}
\noindent
\textit{Proof of Proposition \ref{prop:num_of_iter-RLMC}.}
Given an error tolerance
$\epsilon> 0$, one can derive and thanks to Theorem \ref{thm:main_thm_for_RLMC}, one can choose $k$ to be large enough and $h$ to be small enough such that
\begin{equation}
\label{Randomization-eq:p-estimate-of-error-tolerance}
\begin{aligned}
C_2^{\textcolor{black}{\ell}}
\sqrt{d}
e^{-\lambda k h}
\leq  
\tfrac{\epsilon}{2},
\quad
C_1^{\textcolor{black}{\ell}}
\sqrt{d}
h
\leq  
\tfrac{\epsilon}{2}
,
\end{aligned}
\end{equation}
ensuring 
\begin{equation}
\mathcal{W}_2 
\big( \nu q_k , \pi
\big)
\leq 
\epsilon.
\end{equation}
Solving the first term of inequality  \eqref{Randomization-eq:p-estimate-of-error-tolerance} shows
\begin{equation}
\label{Randomization-eq:p-estimate-k}
k  
\geq  
\tfrac{1}{\lambda  h}
\log \Big( \tfrac{2C_2^{\textcolor{black}{\ell}}
\sqrt{d}}{\epsilon}
\Big).
\end{equation}
The second part of inequality \eqref{Randomization-eq:p-estimate-of-error-tolerance} requires
\begin{equation}
\tfrac{1}{h}
\geq 
\tfrac{2C_1^{\textcolor{black}{\ell}}
\sqrt{d}}{\epsilon}.
\end{equation}
Inserting this into 
\eqref{Randomization-eq:p-estimate-k} yields 
\begin{equation}
k 
\geq 
\tfrac{1}{\lambda }
\cdot
\tfrac{2C_1^{\textcolor{black}{\ell}} \sqrt{d}}{\epsilon}
\cdot
\log \Big( \tfrac{2C_2^{\textcolor{black}{\ell}} \sqrt{d}}{\epsilon}
\Big)
=
\Tilde{O} 
\big( \tfrac{\sqrt{d}}{\epsilon}
\big).
\end{equation}
Thus, we complete the proof.
\qed

\section{Proofs of main results for modified randomized \textcolor{black}{(splitting)} Langevin Monte Carlo}
\label{Randomization-sec:app:proofs_of_non_Lip_case}

We first present some useful properties of the modified sampling algorithm \eqref{Randomization-eq:pRLMC}. 

\begin{lemma}
\label{lem:properties_of_projected_operator_T}
For any $x,y  \in \R^d$, the operator $\cT^h$ defined by \eqref{Randomization-eq:defn-projection-oprator} satisfies the following properties:
\begin{equation}
\label{Randomization-eq:bound_of_cT_and_FcT}
\begin{aligned}
\big|
\cT^h  (x)
\big|
\leq  
\vartheta d^{\frac{1}{2\gamma+2}}
h^{-\frac{1}{2\gamma+2}},
\quad 
\big|
F ( \cT^h (x)  )
\big|
\leq  
L_2'd^{\frac{1}{2}} 
+
2 L_2 d^{\frac{1}{2}}
h^{-\frac{1}{2}},
\end{aligned}
\end{equation}
\begin{equation}
\label{Randomization-eq:difference_between_X_and_operatorTX}
\big|
x
-
\cT^h   (x)
\big|
\leq   
2 
\vartheta^{-2k(\gamma+1)} d^{-k}
h^{k}
|  x |^{2k(\gamma+1)+1} ,
\quad 
\forall 
\,  k \in \N,
\end{equation}
\begin{equation}
\label{Randomization-eq:contraction_of_projected_operator_T}
\big|
\cT^h  (x)
-
\cT^h  (y)
\big|
\leq   
|  x   -   y  |,
\end{equation}
\begin{equation}
\label{Randomization-eq:Lip_of_F_cT}
\big|
F ( \cT^h (x)  )
-
F  (  \cT^h(y)  )
\big|
\leq  
3 L_{2} 
\vartheta^{\gamma} d^{\frac{\gamma}{2\gamma+2}}
h^{-\frac{\gamma}{2\gamma+2}}
| x  -  y|.
\end{equation}
\end{lemma}

The proof is straightforward and omitted here. Similar assertions can be found in Lemma 3.3 and Lemma 5.2 of \cite{pang2023projected}
(See also \cite{yang2025nonasymptotic}). Since $\cT^h(0)=0$, we have
\begin{equation}
\label{Randomization-eq:|Tx|leqx}
|  \cT^h (x) | \leq  |x|, 
\quad 
\forall x \in \R^d.
\end{equation}

\subsection{Proof of Proposition \ref{prop:uniform_in_time_bounded_moments_to_pRLMC}}

\noindent
\textit{Proof of Proposition \ref{prop:uniform_in_time_bounded_moments_to_pRLMC}.}    
We first recast the algorithm \eqref{Randomization-eq:pRLMC} as:
\begin{equation}
\Bar{Y}_{n+1}
=  
\cT^h
(   \Bar{Y}_n   ) 
+
F 
(   \cT^h( 
\Bar{Y}_n  )    )
h 
+ 
\sqrt{2} 
\Delta W_{n+1}
+
\big(
F 
(   \cT^h(
\Bar{Y}_{n+1}^\tau
)  )
-
F (    \cT^h(  \Bar{Y}_n  )  )
\big)
h,
\quad 
n \in \N_0.
\end{equation}
Taking square on both sides and using the Cauchy-Schwarz inequality show
\begin{equation}
\label{Randomization-eq:pri_estimate_pRLMC_mean_square}
\begin{aligned}
\big| \Bar{Y}_{n+1} \big|^2
=   &
\big| \cT^h
(   \Bar{Y}_n   ) \big|^2
+
h^2 
\big| F 
(   \cT^h( 
\Bar{Y}_n  )    )  \big|^2
+ 
2 
\big| \Delta W_{n+1} \big|^2
+
h^2
\big|
F 
(   \cT^h(
\Bar{Y}_{n+1}^\tau
)  )
-
F (    \cT^h(  \Bar{Y}_n  )  )
\big|^2
\\  & 
+
2h 
\big\langle
\cT^h
(   \Bar{Y}_n   ) ,
 F 
(   \cT^h( 
\Bar{Y}_n  )  
\big\rangle
+
2 \sqrt{2} 
\big\langle
\cT^h( 
\Bar{Y}_n ) ,
\Delta W_{n+1}
\big\rangle
\\  & 
+
2h 
\big\langle
\cT^h
(   \Bar{Y}_n   ) ,
F 
(   \cT^h(
\Bar{Y}_{n+1}^\tau
)  )
-
F (    \cT^h(  \Bar{Y}_n  )  )
\big\rangle
+
2 \sqrt{2} h 
\big\langle
 F 
(   \cT^h( 
\Bar{Y}_n  ) ,
\Delta W_{n+1}
\big\rangle
\\   &
+
2  h^2
\big\langle
 F 
(   \cT^h( 
\Bar{Y}_n  ) ,
F 
(   \cT^h(
\Bar{Y}_{n+1}^\tau
)  )
-
F (    \cT^h(  \Bar{Y}_n  )  )
\big\rangle
\\  & 
+
2 \sqrt{2} h 
\big\langle
\Delta W_{n+1},
F 
(   \cT^h(
\Bar{Y}_{n+1}^\tau
)  )
-
F (    \cT^h(  \Bar{Y}_n  )  )
\big\rangle
\\ 
\leq   &
\big(
1+\tfrac{\mu h }{2}
\big)
\big| \cT^h
(   \Bar{Y}_n   ) \big|^2
+
3 h^2 
\big| F 
(   \cT^h( 
\Bar{Y}_n  )    )  \big|^2
+ 
6 
\big| \Delta W_{n+1} \big|^2
\\  & 
+
\big(
3 h^2
+
\tfrac{2h}{\mu}
\big)
\big|
F 
(   \cT^h(
\Bar{Y}_{n+1}^\tau
)  )
-
F (    \cT^h(  \Bar{Y}_n  )  )
\big|^2
\\  &
+
2h 
\big\langle
\cT^h
(   \Bar{Y}_n   ) ,
 F 
(   \cT^h( 
\Bar{Y}_n  )  
\big\rangle
+
2 \sqrt{2} 
\big\langle
\cT^h( 
\Bar{Y}_n ) ,
\Delta W_{n+1}
\big\rangle
.
\end{aligned}
\end{equation}
Before proceeding further, we employ \eqref{Randomization-eq:bound_of_cT_and_FcT} to arrive at 
\begin{equation}
\label{Randomization-eq:pri_estimate_pRLMC_second_term}
3 h^2 
\big| F 
(   \cT^h( 
\Bar{Y}_n  )    )  \big|^2
\leq  
6 L_2'^2 d h^2  +  24 L_2^2 d  h
\leq 
\big( 6 L_2'^2   +  24 L_2^2 
\big) d  h .
\end{equation}
Thanks to \eqref{Randomization-eq:Lip_of_F_cT} and the assumption $h\leq d^{-\gamma}$, one can easily see 
\begin{equation}
\begin{aligned}
\big(
3 h^2
+
\tfrac{2h}{\mu}
\big)
\big|
F 
(   \cT^h(
\Bar{Y}_{n+1}^\tau
)  )
-
F (    \cT^h(  \Bar{Y}_n  )  )
\big|^2
\leq   &
\big(
3 
+
\tfrac{2}{\mu}
\big)
h
\big|
F 
(   \cT^h(
\Bar{Y}_{n+1}^\tau
)  )
-
F (    \cT^h(  \Bar{Y}_n  )  )
\big|^2
\\ 
\leq  & 
9\big(
3 
+
\tfrac{2}{\mu}
\big)
L_2^2 
\vartheta^{2\gamma} d^{\frac{2\gamma}{2\gamma+2}}
h^{1-\frac{2\gamma}{2\gamma+2}}
\big|
\Bar{Y}_{n+1}^\tau
-
\Bar{Y}_n  
\big|^2
\\  
\leq   &
9\big(
3 
+
\tfrac{2}{\mu}
\big)
L_2^2 
\vartheta^{2\gamma} 
\big|
\Bar{Y}_{n+1}^\tau
-
\Bar{Y}_n  
\big|^2.
\end{aligned}
\end{equation}
In view of \eqref{Randomization-eq:pRLMC}, we use the Cauchy-Schwarz inequality to acquire
\begin{equation}
\begin{aligned}
\big|
\Bar{Y}_{n+1}^\tau
-
\Bar{Y}_n  
\big|^2
\leq   &
2 |\tau_{n+1}|^2 \textcolor{black}{\ell^2} h^2  
\big| F 
(   \cT^h( 
\Bar{Y}_n  )    )  \big|^2 
+
4  \big|  \Delta W_{n+1}^\tau
\big|^2 
\\  
\leq    &
\big( 4 \textcolor{black}{\ell^2}  L_2'^2   +  16 \textcolor{black}{\ell^2}  L_2^2 
\big) d  h 
+
4  \big|  \Delta W_{n+1}^\tau
\big|^2 .
\end{aligned}
\end{equation}
Thus, we get 
\begin{equation}
\label{Randomization-eq:pri_estimate_pRLMC_fourth_term}
\big(
3 h^2
+
\tfrac{2h}{\mu}
\big)
\big|
F 
(   \cT^h(
\Bar{Y}_{n+1}^\tau
)  )
-
F (    \cT^h(  \Bar{Y}_n  )  )
\big|^2
\leq  
C_F \textcolor{black}{\ell^2} 
\big(   L_2'^2    +  4 L_2^2
\big) d h
+
C_F
\big|  \Delta W_{n+1}^\tau
\big|^2 ,
\end{equation}
where $C_F:=36\big(
3 
+
\tfrac{2}{\mu}
\big)
L_2^2 
\vartheta^{2\gamma} $.
Using \eqref{Randomization-eq:diss_cond}, we have 
\begin{equation}
\label{Randomization-eq:pri_estimate_pRLMC_fifth_term}
2h 
\big\langle
\cT^h
(   \Bar{Y}_n   ) ,
 F 
(   \cT^h( 
\Bar{Y}_n  )  
\big\rangle
\leq  
-2 \mu h 
\big|
\cT^h
(   \Bar{Y}_n   )
\big|^2
+
2 \mu' d h.
\end{equation}
Equipped with  estimates \eqref{Randomization-eq:pri_estimate_pRLMC_second_term}, \eqref{Randomization-eq:pri_estimate_pRLMC_fourth_term} and \eqref{Randomization-eq:pri_estimate_pRLMC_fifth_term}, one can derive from  \eqref{Randomization-eq:pri_estimate_pRLMC_mean_square}  that 
\begin{equation}
\begin{aligned}
\big| \Bar{Y}_{n+1} \big|^2
\leq   &
\big(
1-\tfrac{3\mu h }{2}
\big)
\big| \cT^h
(   \Bar{Y}_n   ) \big|^2
+
2 \sqrt{2} 
\big\langle
\cT^h( 
\Bar{Y}_n ) ,
\Delta W_{n+1}
\big\rangle
+ 
6 
\big| \Delta W_{n+1} \big|^2
+
C_F 
\big| \Delta W_{n+1}^\tau \big|^2
\\  &
+
\big(  6 L_2'^2    +  24 L_2^2
\big) d h
+
C_F
\big(  L_2'^2    +    4 L_2^2
\big) d h
+
2 \mu'd h
\\  
\leq  &
\big(
1-\tfrac{3\mu h }{2}
\big)
\big| \cT^h
(   \Bar{Y}_n   ) \big|^2
+
2 \sqrt{2} 
\big\langle
\cT^h( 
\Bar{Y}_n ) ,
\Delta W_{n+1}
\big\rangle
+ 
6 
\big| \Delta W_{n+1} \big|^2
\\   &
+
C_F 
\big| \Delta W_{n+1}^\tau \big|^2
+
C_M
d  h
 \\  
=:   &
\big(
1-\tfrac{3\mu h }{2}
\big)
\big| \cT^h
(   \Bar{Y}_n   ) \big|^2  
+
\Xi_{n+1} ,
\end{aligned}
\end{equation}
where $C_M^{\textcolor{black}{\ell} }:=
(6  + C_F \textcolor{black}{\ell^2} )
(  L_2'^2    +  4 L_2^2)
+
2 \mu'$ 
and for short we denote 
\begin{equation}
\Xi_{n+1} 
:=
2 \sqrt{2} 
\big\langle
\cT^h( 
\Bar{Y}_n ) ,
\Delta W_{n+1}
\big\rangle
+ 
6 
\big| \Delta W_{n+1} \big|^2
+
C_F 
\big| \Delta W_{n+1}^\tau \big|^2
+
C_M^{\textcolor{black}{\ell} }
d  h.
\end{equation}
Taking $p$-th power for $p\in \N$ and then expectations, 
we further use the binomial expansion theorem to deduce
\begin{equation}
\label{Randomization-eq:pri_pRLMC_moments}
\begin{aligned}
\E
\Big[
\big| \Bar{Y}_{n+1} \big|^{2p}
\Big]
\leq    &
\big(
1-\tfrac{3\mu h }{2}
\big)^{p}
\E
\Big[
\big| \cT^h
(   \Bar{Y}_n   ) \big|^{2p}
\Big]
+
\sum_{k=1}^{p}
C_k^{p}
\big(
1-\tfrac{3\mu h }{2}
\big)^{p-k}
\E
\Big[
\big| \cT^h
(   \Bar{Y}_n   ) \big|^{2p-2k}
(\Xi_{n+1})^k
\Big],
\end{aligned}
\end{equation}
where $C_k^{p}:=\frac{p!}{k!(p-k)!}$.
Now, we estimate the second term for two cases: $k=1$ and $k\geq 2$.
We first notice that 
%
%
$
| \cT^h
(   \Bar{Y}_n   ) |^{2p-2k}$ is $\CF_{t_n}$-measurable.
By further taking conditional expectation with respect to $\CF_{t_n}$, one can see that 
\begin{equation}
\begin{aligned}
\E
\Big[
\big| \cT^h
(   \Bar{Y}_n   ) \big|^{2p-2k}
(\Xi_{n+1})^k
\Big]
=  &
\E
\Big[
\E
\Big[
\big| \cT^h
(   \Bar{Y}_n   ) \big|^{2p-2k}
(\Xi_{n+1})^k
\Big| 
\CF_{t_n}
\Big]
\Big]
\\
=   &
\E
\Big[
\big| \cT^h
(   \Bar{Y}_n   ) \big|^{2p-2k}
\E
\Big[
(\Xi_{n+1})^k
\Big| 
\CF_{t_n}
\Big]
\Big].
\end{aligned}
\end{equation}
Recall some properties of the Gaussian random variable: for any $q \in \N $ and for any $0 \le s \le t$,
\begin{equation}
\label{Randomization-eq:power_properties_of_the_Gaussian random}
\E_W
\Big[
\big| W_t^i-W_s^i \big|^{2q}
\Big|
\CF^W_{s}
\Big]
=
(2q-1)!!(t-s)^q,
\quad 
\E_W
\Big[
\big( W_t^i-W_s^i \big)^{2q-1}
\Big|
\CF^W_{s}
\Big]
=0 ,
\quad 
i \in  [d].
\end{equation}
With regard to $k=1$, we thus have  
\begin{equation}
\begin{aligned}
\E
\Big[
\Xi_{n+1}
\Big| 
\CF_{t_n}
\Big]
=   &
2 \sqrt{2} 
\Big\langle
\cT^h( 
\Bar{Y}_n )
,
\E
\Big[
\Delta W_{n+1}
\Big| 
\CF_{t_n}
\Big]
\Big\rangle
+ 
6 
\E
\Big[
\big| \Delta W_{n+1} \big|^2
\Big|
\CF_{t_n}
\Big]
\\  &
+
C_F 
\E
\Big[
\big| \Delta W_{n+1}^\tau \big|^2
\Big|
\CF_{t_n}
\Big]
+
C_M^{\textcolor{black}{\ell} }
d  h
\\  
=  &
6 d h 
+
C_F d h  
+
C_M^{\textcolor{black}{\ell} } d h.
\end{aligned}
\end{equation}
Therefore, we get 
\begin{equation}
\label{Randomization-eq:k=1_of_second_term}
\begin{aligned}
&C_k^{1}
\big(
1-\tfrac{3\mu h }{2}
\big)^{p-1}
\E
\Big[
\big| \cT^h
(   \Bar{Y}_n   ) \big|^{2p-2}
\Xi_{n+1}
\Big]
\\
=   &
C_k^{1}
\big(
1-\tfrac{3\mu h }{2}
\big)^{p-1}
(6d +C_Fd +C_M^{\textcolor{black}{\ell} } d)h
\E
\Big[
\big| \cT^h
(   \Bar{Y}_n   ) \big|^{2p-2}
\Big]
\\
\leq    &
C_{\Xi,1}  d h 
\E
\Big[
\big| \cT^h
(   \Bar{Y}_n   ) \big|^{2p-2}
\Big]
\\  
\leq  &
\varepsilon_1 h
\E
\Big[
\big| \cT^h
(   \Bar{Y}_n   ) \big|^{2p}
\Big]
+
C(\varepsilon_1)
(C_{\Xi,1})^{p} d^{p}  h,
\end{aligned}
\end{equation}
where the last step stands due to the Young inequality with $\varepsilon_1>0$,   $C(\varepsilon_1):=
\frac{1}{p}(\frac{\varepsilon_1p}{p-1})^{p-1}$,  
$C_{\Xi,1} $ is a dimension-independent constant, depending on $\mu, \mu',\vartheta,p,\textcolor{black}{\ell},L_2,L_2'$.

For $k\geq 2$, using a fundamental inequality shows
\begin{equation}
\begin{aligned}
(\Xi_{n+1})^k
\leq    &
4^{k-1}
\Big(
2^{\frac{3k}{2}} 
\big\langle
\cT^h( 
\Bar{Y}_n ) ,
\Delta W_{n+1}
\big\rangle^k
+ 
6^k 
\big| \Delta W_{n+1} \big|^{2k}
+
C_F^k d^k 
\big| \Delta W_{n+1}^\tau \big|^{2k}
+
(C_M^{\textcolor{black}{\ell} })^k
d^{k}  h^k
\Big)
\\   
\leq   &
C
\Big(
\big|
\cT^h( 
\Bar{Y}_n )
\big|^{k}
\big|
\Delta W_{n+1} 
\big|^{k}
+
\big| \Delta W_{n+1} \big|^{2k}
+
\big| \Delta W_{n+1}^\tau \big|^{2k}
+
d^{k}
h^k
\Big),
\end{aligned}
\end{equation}
where $C$ depends on $\mu, \mu',\vartheta,p,\textcolor{black}{\ell},L_2,L_2'$.
Keeping this in mind, one can derive from \eqref{Randomization-eq:power_properties_of_the_Gaussian random} that
\begin{equation}
\begin{aligned}
& \E
\Big[
(\Xi_{n+1})^k
\Big| 
\CF_{t_n}
\Big]
\\  
\leq   &
C
\bigg(
\big|
\cT^h( 
\Bar{Y}_n )
\big|^{k}
\E
\Big[
\big|
\Delta W_{n+1} 
\big|^{k}
\Big| 
\CF_{t_n}
\Big]
+
\E
\Big[
\big|
\Delta W_{n+1} 
\big|^{2k}
\Big| 
\CF_{t_n}
\Big]
+
\E
\Big[
\big|
\Delta W_{n+1}^\tau
\big|^{2k}
\Big| 
\CF_{t_n}
\Big]
+
d^{k}
h^k
\bigg)
\\ 
\leq   &
C
\bigg(
(k-1)!!
d^{\frac{k}{2}}
h^{\frac{k}{2}}
\big|
\cT^h( 
\Bar{Y}_n )
\big|^{k}
+
(2k-1)!! d^{k} h^{k}
+
(2k-1)!! d^{k} h^{k}
+
d^{k}
h^k
\bigg).
\end{aligned}
\end{equation}
So, we get, for $k \geq 2$,
\begin{equation}
\begin{aligned}
&C_k^{p}
\big(
1-\tfrac{3\mu h }{2}
\big)^{p-k}
\E
\Big[
\big| \cT^h
(   \Bar{Y}_n   ) \big|^{2p-2k}
(\Xi_{n+1})^k
\Big]
\\   
\leq      & 
C_k^{p} 
C
\big(
1-\tfrac{3\mu h }{2}
\big)^{p-k}
d^{\frac{k}{2}}
h^{\frac{k}{2}}
\E
\Big[
\big| \cT^h
(   \Bar{Y}_n   ) \big|^{2p-k}
\Big]
+
C_k^{p} 
C
\big(
1-\tfrac{3\mu h }{2}
\big)^{p-k}
d^{k} h^{k}
\E
\Big[
\big| \cT^h
(   \Bar{Y}_n   ) \big|^{2p-2k}
\Big]
\\  
\leq   &
C_{\Xi,2}
d^{\frac{k}{2}}
h
\E
\Big[
\big| \cT^h
(   \Bar{Y}_n   ) \big|^{2p-k}
\Big]
+
C_{\Xi,3}
d^{k} h
\E
\Big[
\big| \cT^h
(   \Bar{Y}_n   ) \big|^{2p-2k}
\Big],
\end{aligned}
\end{equation}
where $C_{\Xi,2} $ and $C_{\Xi,3} $ are also two dimension-independent constants, depending on $\mu, \mu',\vartheta,p,\textcolor{black}{\ell},L_2,L_2'$.
Again, using the Young inequality implies 
\begin{equation}
C_{\Xi,2}
d^{\frac{k}{2}}
h
\E
\Big[
\big| \cT^h
(   \Bar{Y}_n   ) \big|^{2p-k}
\Big]
\leq  
\varepsilon_2 h 
\E
\Big[
\big| \cT^h
(   \Bar{Y}_n   ) \big|^{2p}
\Big]
+
C(\varepsilon_2)
(C_{\Xi,2})^{p} d^{p}h,
\end{equation}
\begin{equation}
C_{\Xi,3}
d^{k} h
\E
\Big[
\big| \cT^h
(   \Bar{Y}_n   ) \big|^{2p-2k}
\Big]
\leq  
\varepsilon_3 h 
\E
\Big[
\big| \cT^h
(   \Bar{Y}_n   ) \big|^{2p}
\Big]
+
C(\varepsilon_3)
(C_{\Xi,3})^{p} d^{p}h
,
\end{equation}
where $C(\varepsilon_2):=\frac{k}{2p}(\frac{\varepsilon_2p}{p-k/2})^{2p/k-1}$ and 
$C(\varepsilon_3):=\frac{k}{p}(\frac{\varepsilon_3p}{p-k})^{p/k-1}$.
This immediately implies,  
\begin{equation}
\label{Randomization-eq:k>1_of_second_term}
\begin{aligned}
&C_k^{p}
\big(
1-\tfrac{3\mu h }{2}
\big)^{p-k}
\E
\Big[
\big| \cT^h
(   \Bar{Y}_n   ) \big|^{2p-2k}
(\Xi_{n+1})^k
\Big]  
\\  
\leq    &
(\varepsilon_2+\varepsilon_3) h 
\E
\Big[
\big| \cT^h
(   \Bar{Y}_n   ) \big|^{2p}
\Big]
+
C(\varepsilon_2)
(C_{\Xi,2})^{p} d^{p}h
+
C(\varepsilon_3)
(C_{\Xi,3})^{p} d^{p}h.
\end{aligned}
\end{equation}
Inserting this and 
\eqref{Randomization-eq:k=1_of_second_term} into the second term of 
\eqref{Randomization-eq:pri_pRLMC_moments}, we have 
\begin{equation}
\begin{aligned}
& \sum_{k=1}^{p}
C_k^{p}
\big(
1-\tfrac{3\mu h }{2}
\big)^{p-k}
\E
\Big[
\big| \cT^h
(   \Bar{Y}_n   ) \big|^{2p-2k}
(\Xi_{n+1})^k
\Big]
\\  
\leq   &
\big((\varepsilon_1
+
(p-1)
(\varepsilon_2+\varepsilon_3)
\big)) h 
\E
\Big[
\big| \cT^h
(   \Bar{Y}_n   ) \big|^{2p}
\Big]
\\   &
+
C(\varepsilon_1)
(C_{\Xi,1})^{p} d^{p}  h
+
C(\varepsilon_2)
(C_{\Xi,2})^{p} d^{p}h
+
C(\varepsilon_3)
(C_{\Xi,3})^{p} d^{p}h.
\end{aligned}
\end{equation}
By setting $\varepsilon_1=\frac{\mu h}{p}$ and $\varepsilon_2=\varepsilon_3=\frac{(p-1)\mu h}{2p}$, one can easily see 
\begin{equation}
\sum_{k=1}^{p}
C_k^{p}
\big(
1-\tfrac{3\mu h }{2}
\big)^{p-k}
\E
\Big[
\big| \cT^h
(   \Bar{Y}_n   ) \big|^{2p-2k}
(\Xi_{n+1})^k
\Big]
\leq  
\mu h 
\E
\Big[
\big| \cT^h
(   \Bar{Y}_n   ) \big|^{2p}
\Big]
+
\CM_3
d^{p}h,
\end{equation}
where $\CM_3^{\textcolor{black}{\ell}}$ is a dimension-independent constant, depending on $\mu, \mu',\vartheta,p,\textcolor{black}{\ell},L_2,L_2'$.
Putting this into \eqref{Randomization-eq:pri_pRLMC_moments}, one can use the inequality $\big(
1-\frac{3\mu h }{2}
\big)^{p}\leq  1-\frac{3\mu h }{2}$, $p\geq 1$ to obtain 
\begin{equation}
\begin{aligned}
\E
\Big[
\big| \Bar{Y}_{n+1} \big|^{2p}
\Big]
\leq      &
\big(
1-\tfrac{3\mu h }{2}
\big)^{p}
\E
\Big[
\big| \cT^h
(   \Bar{Y}_n   ) \big|^{2p}
\Big] 
+
\mu h 
\E
\Big[
\big| \cT^h
(   \Bar{Y}_n   ) \big|^{2p}
\Big]
+
\CM_3^{\textcolor{black}{\ell}}
d^{p}h
\\   
\leq   &
\big(
1-\tfrac{\mu h }{2}
\big)
\E
\Big[
\big| \cT^h
(   \Bar{Y}_n   ) \big|^{2p}
\Big] 
+
\CM_3^{\textcolor{black}{\ell}}
d^{p}h
\\   
\leq   &
\big(
1-\tfrac{\mu h }{2}
\big)
\E
\Big[
\big| \Bar{Y}_n   \big|^{2p}
\Big] 
+
\CM_3^{\textcolor{black}{\ell}}
d^{p}h,
\end{aligned}
\end{equation}
where we used \eqref{Randomization-eq:|Tx|leqx} in the last step.
By iteration, we acquire 
\begin{equation}
\begin{aligned}
\E
\Big[
\big| \Bar{Y}_{n+1} \big|^{2p}
\Big]
\leq      &
\big(
1-\tfrac{\mu h }{2}
\big)^{n+1}
\E
\big[
| x_0 |^{2p}
\big]  
+
\CM_3^{\textcolor{black}{\ell}}
d^{p}h 
\sum_{i=1}^{n}
\big(
1-\tfrac{\mu h }{2}
\big)^{i}
\\ 
\leq   &
e^{-\frac{\mu t_{n+1}}{2}} 
\E
\big[
| x_0 |^{2p}
\big]  
+
\tfrac{2\CM_3^{\textcolor{black}{\ell}}
d^{p}}{\mu},
\end{aligned}
\end{equation}
where the inequality $1-u \leq  e^{-u}, u>0 $ was also employed.
We thus complete the proof.
\qed

\subsection{Proof of Proposition \ref{prop:finite_time_error_analysis_pRLMC}}
The aim of this subsection is to prove the finite-time convergence of the algorithm \eqref{Randomization-eq:pRLMC}, by utilizing the mean-square fundamental theorem in \cite{yang2025nonasymptotic}.
To this end, 
we first present some auxiliary lemmas that will be used to prove Proposition \ref{prop:finite_time_error_analysis_pRLMC}.

\begin{lemma}
\label{lem:holder_continuous_of_LSDE_non_lip}
Let 
Assumption \ref{Randomization-ass:poly_growth_condition} be fulfilled. Let $X(s,x;t)$  denote the solution to the  Langevin SDE \eqref{Randomization-eq:langevin_SDE} at $t$, starting from the initial value $x$ at $s$. 
If the uniform stepsize  $h>0$ satisfies $h \leq 
1 \wedge \tfrac{1}{2L_2}\wedge\tfrac{1}{L_2'}$,
then, 
for any $x\in \R^d$, any $0 <\theta \leq h$ and $0\leq s  \leq t$, it holds that
\begin{equation}
\label{Randomization-eq:Holder_continuous_of_X_t}
\begin{aligned}
\E_W
\big[
| 
X(s,x;t+\theta ) 
-  
X(s,x;t ) 
|^{2p}
\big]
\leq   &
\Big(
\cH_1 (p)
d^{p(\gamma+1)}
+
\cH_2 (p)
|
x
|^{2p(\gamma+1)}
\Big)
\theta^{p},
\end{aligned}
\end{equation}
where $\cH_1 (p):=
2^{4p-2}
L_2'^{p}
+
2^{4p-2}
\CM_1(p(\gamma+1))
+
2^{3p-1} (2p-1)!!
$ and 
$\cH_2 (p):=
2^{5p-2} 
L_2^{p} $.
Moreover, there exist two dimension-independent constants $\cH_1^F
$ and 
$\cH_2^F $ such that,  for any $x\in \R^d$, any $0 <\theta \leq h$ and $0\leq s  \leq t$, 
\begin{equation}
\label{Randomization-eq:Holder_continuous_of_F_X_t}
\E_W
\big[
\big|
F  (X(s,x;t+\theta ) )
-
F  (X(s,x;t ) )
\big|^2
\big]
\leq 
\Big(
\cH^F_1
d^{2\gamma+1}
+
\cH^F_2
|
x
|^{4\gamma+2}
\Big)
\theta, 
\end{equation}
where $\gamma>0$ comes from \eqref{Randomization-eq:poly_growth_cond}.
\end{lemma}
\begin{proof}
Using similar arguments as used in \eqref{Randomization-eq:estimate_holder_continuous_X_t}, and employing \eqref{Randomization-eq:poly_growth_cond1}, we have, for any $x\in \R^d$, any $0 <\theta \leq h$ and $0\leq s  \leq t$,
\begin{equation}
\begin{aligned}
& \E_W
\big[
|
X(s,x;t+\theta ) 
-  
X(s,x;t )
|^{2p}
\big]
\\ 
\leq   &
2^{2p-1}
\E_W
\bigg[
\bigg|
\int_t^{t+\theta} 
F (X(s,x;r ))\,
\dd r
\bigg|^{2p}
+
\bigg|
\int_t^{t+\theta}
\sqrt{2}
\,
\dd W_r
\bigg|^{2p}
\bigg]
\\ 
\leq  &
2^{2p-1}
\theta^{p} 
\bigg(
\theta^{p-1}
\int_t^{t+\theta}
\E_W
\big[
\big|
F (X(s,x;r ))
\big|^{2p}
\big]\,
\dd r
+
2^p (2p-1)!! d^p  
\bigg) 
\\ 
\leq  & 
2^{2p-1} \theta^{p}
\bigg(
2^{2p-1}
\theta^{p-1}
\int_t^{t+\theta}
\Big(
2^{2p}L_2^{2p}
\E_W
\big[
\big|
X(s,x;r )
\big|^{2p(\gamma + 1)}
\big]
+
L_2'^{2p}
d^p 
\Big)
\,
\dd r
+
2^p (2p-1)!! d^p 
\bigg)
\\ 
\leq  & 
2^{2p-1} \theta^{p}
\bigg(
2^{2p-1}
\theta^{p-1}
\int_t^{t+\theta}
\Big(
2^{2p}L_2^{2p}
|
x
|^{2p(\gamma + 1)}
+
L_2'^{2p}
d^p 
+
\CM_1 (p(\gamma+1)) d^{p(\gamma+1)}
\Big)
\,
\dd r
+
2^p (2p-1)!! d^p  
\bigg)
\\ 
\leq  &
\bigg(
\Big(
2^{4p-2}
L_2'^{p}
+
2^{4p-2}
\CM_1(p(\gamma+1))
+
2^{3p-1} (2p-1)!!
\Big)
d^{p(\gamma+1)}
+
2^{5p-2} 
L_2^{p} 
|
x
|^{2p(\gamma+1)}
\bigg)
\theta^{p},
\end{aligned}
\end{equation}
where the fourth step holds true due to Lemma \ref{prop:uniform_in_time_bounded_moments_to_LSDE}.
Now, we estimate \eqref{Randomization-eq:Holder_continuous_of_F_X_t}.
Again,
thanks to $h \leq 1 \wedge \tfrac{1}{2L_2} \wedge \tfrac{1}{L_2'} $ and Lemma \ref{prop:uniform_in_time_bounded_moments_to_LSDE},
using \eqref{Randomization-eq:Holder_continuous_of_X_t},  the H\"older inequality  and polynomial growth condition \eqref{Randomization-eq:poly_growth_cond} yields, for any $x\in \R^d$, any $0 <\theta \leq h$ and $0\leq s  \leq t$,
\begin{equation}
\begin{aligned}
& \E_W
\big[
\big|
F  (X(s,x;t+\theta ) )
-
F  (X(s,x;t ) )
\big|^2
\big]
\\ 
\leq   &
L_1^2 \,
\E_W
\Big[
\big(
1 
+ 
| X(s,x;t+\theta ) |^{\gamma}
+
| X(s,x;t ) |^{\gamma}
\big)^2
\big|
X(s,x;t+\theta )
-
X(s,x;t )
\big|^2
\Big]
\\ 
\leq   &
L_2^2 
\bigg(
\E_W
\Big[
\big(
1 
+ 
| X(s,x;t+\theta ) |^{\gamma}
+
| X(s,x;t) |^{\gamma}
\big)^{\frac{4\gamma +2}{\gamma} }
\Big]
\bigg)^{\frac{\gamma}{2\gamma+1}}
\bigg(
\E_W
\Big[
\big|
X(s,x;t+\theta )
-
X(s,x;t )
\big|^{\frac{4\gamma+2}{\gamma+1}
}
\Big]
\bigg)^{\frac{\gamma+1}{2\gamma+1}}
\\ 
\leq  & 
9
L_2^2  
\theta
\bigg(
\E_W
\Big[
1 
+ 
| X(s,x;t+\theta ) |^{4\gamma +2}
+
| X(s,x;t ) |^{4\gamma +2}
\Big]
\bigg)^{\frac{\gamma}{2\gamma+1}}
\bigg(
\cH_1(\tfrac{2\gamma+1}{\gamma+1})
d^{2\gamma+1}
+
\cH_1(\tfrac{2\gamma+1}{\gamma+1})
\big|
x
\big|^{4\gamma+2}
\bigg)^{\frac{\gamma+1}{2\gamma+1}}
\\  
\leq   &
C(\gamma) L_2^2 
\theta
\bigg(
\CM_1 ( 2\gamma+1) 
d^{2\gamma+1}
+
|  x |^{4\gamma +2}
\Big]
\bigg)^{\frac{\gamma}{2\gamma+1}}
\bigg(
\cH_1(\tfrac{2\gamma+1}{\gamma+1})
d^{2\gamma+1}
+
\cH_1(\tfrac{2\gamma+1}{\gamma+1})
\big|
x
\big|^{4\gamma+2}
\bigg)^{\frac{\gamma+1}{2\gamma+1}}
\\  
\leq    &
\Big(
\cH^F_1
d^{2\gamma+1}
+
\cH^F_2
|
x
|^{4\gamma+2}
\Big)
\theta.
\end{aligned}
\end{equation}
Here $\cH_1^F
$ and 
$\cH_2^F 
$  are two dimension-independent constants, depending on $c, \mu, \mu', \gamma, L_2,L_2'$.
\end{proof}

Also, we need to introduce the related one-step approximation,  defined by, for any $t\in[0,+\infty)$,  $\tau \sim \mathcal{U}(0,1)$, $h \in (0,1)$ and  $x\in \R^d$,
\begin{equation}
\label{Randomization-eq:one_step_pRLMC}
\begin{aligned}
&\Bar{Y}_m (t,x;t+\tau h)
:=  
x 
+
\textcolor{black}{\ell}
F ( \cT^h(x) )
\tau h 
+
\sqrt{2}
(
W_{t+\tau h }
-
W_t),
\\
 & \Bar{Y}(t,x;t+h)
:= 
\cT^h(x)
+
F (   \cT^h ( \Bar{Y}_m(t,x;t+ \tau h)
))h
+
\sqrt{2}
(W_{t+ h }
-
W_t),
\end{aligned}
\end{equation}
and the one-step of Langevin dynamics \eqref{Randomization-eq:langevin_SDE}, given by
\begin{equation}
\label{Randomization-eq:one-step_LSDE}
X(t,x;t+h)
=
x 
+
\int_{t}^{t+h}
F   (   X  (t,x;s   )  )
\dd s 
+
\sqrt{2}
(W_{t+ h }
-
W_t).
\end{equation}

Then we are able to show error estimates for the one-step approximations, which are enough for the desired finite-time error estimates.

\begin{lemma}
\label{lem:onestep_waek_and_strong_error_of_pRLMC}
Let Assumptions \ref{Randomization-ass:diss_cond_and_grad_Lip_cond}, \ref{Randomization-ass:poly_growth_condition} hold. Let $X(t,x;t+h)$, defined by \eqref{Randomization-eq:one-step_LSDE}, be the solution to the  Langevin SDE \eqref{Randomization-eq:langevin_SDE} at $t+h$, starting from the initial value $x $ at $t$ and let the uniform stepsize  $h>0$ satisfy $h \leq 
1 \wedge \tfrac{1}{2L_2}\wedge\tfrac{1}{L_2'}$. Then, for all $x \in \R^d$ and $t\in[0, +\infty)$, the one-step approximation \eqref{Randomization-eq:one_step_pRLMC} satisfies
\begin{equation}
\begin{aligned}
\big|
\E 
\big[
X(t,x;t+h)
-
\Bar{Y}(t,x;t+h)
\big]
\big|
\leq    &
\Bar{K}_{1}^{\textcolor{black}{\ell}}
\Big(  
d^{5\gamma+1}
+
d^{-4}
|
x
|^{10\gamma+10}
\Big)^{\frac12}
h^2,
\\  
\Big(
\E 
\big[
\big|
X(t,x;t+h)
-
\Bar{Y}(t,x;t+h)
\big|^2
\big]
\Big)^{\frac12}
\leq    &
\Bar{K}_{2}^{\textcolor{black}{\ell}}
\Big(  
d^{5\gamma+1}
+
d^{-4}
|
x
|^{10\gamma+10}
\Big)^{\frac12}
h^{\frac32},
\end{aligned}
\end{equation}
where $\Bar{K}_1^{\textcolor{black}{\ell}}$ and $\Bar{K}_2^{\textcolor{black}{\ell}}$ are two dimension-independent constants, depending on $\mu,\mu',\gamma,\vartheta,\textcolor{black}{\ell},L_2,L_2'$.
\end{lemma}
\begin{proof}
\textcolor{black}{
First, it follows from \eqref{Randomization-eq:one_step_pRLMC} and \eqref{Randomization-eq:one-step_LSDE} that,  for all $x \in \R^d$ and $t\in[0, +\infty)$,
\begin{equation}
\label{Randomization-eq:difference_onestep_of_pRLMC}
\begin{aligned}
&
X(t,x;t+h)
-
\Bar{Y}(t,x;t+h)
\\
=   &
x - \cT^h (x)
+
\int_{t}^{t+h}
\big(
F   (   X  (t,x;s   )  )
-
F (   \cT^h ( \Bar{Y}_m(t,x;t+ \tau h)
))
\big)
\dd s
\\    
=    &
x - \cT^h (x)
+
\underbrace{
\int_{t}^{t+h}
\big(
F   (   X  (t,x;s   )  )
-
F (  X (t,x;t+ \tau h)
)
\big)
\dd s
}_{=:\mathbb{J}_1}
\\   &
+
\underbrace{
\int_{t}^{t+h}
\big(
F (  X (t,x;t+ \tau h)
)
-
F (      \Bar{Y}_m (t,x;t+ \tau h)
)
\big)
\dd s
}_{= : \mathbb{J}_2}
\\   &
+
\underbrace{
\int_{t}^{t+h}
\big(
F
(   
\Bar{Y}_m (t,x;t+ \tau h)
)
-
F (   
\cT^h (   \Bar{Y}_m (t,x;t+ \tau h))
)
\big)
\dd s
}_{=: \mathbb{J}_3}.
\end{aligned}
\end{equation}
Applying the triangle inequality gives 
\begin{equation}
\label{Randomization-eq:pri-estimate-one-step-weak-error-pRLMC}
\big|
\E 
\big[
X(t,x;t+h)
-
\Bar{Y}(t,x;t+h)
\big]
\big|
\leq 
\big|
x - \cT^h (x)
\big|
+
\big|
\E 
\big[
\mathbb{J}_1
\big]
\big|
+
\big|
\E 
\big[
\mathbb{J}_2
\big]
\big|
+
\big|
\E 
\big[
\mathbb{J}_3
\big]
\big|.
\end{equation}
The above four items will be treated one by one.
In virtue of Lemma \ref{lem:properties_of_projected_operator_T}, we have
\begin{equation}
\label{Randomization-eq:estimate_one-step_weak_first_term}
\big|
x - \cT^h (x)
\big|
\leq  
2\vartheta^{-4(\gamma+1)} d^{-2}
h^{2}
|  x |^{4\gamma+5}.
\end{equation}
Recalling \eqref{Randomization-eq:property_of_uniform-distribution},  one can easily see that, for any $\tau
\sim 
\mathcal{U}(0,1)$, \begin{equation}
\int_{t}^{t+h}
F   (   X  (t,x;s   )  )
\dd s
=
h
\int_{0}^{1}
F   (   X  (t,x; t+sh   )  )
\dd s
=
h
\E_\tau  
\big[
F   (   X  (t,x; t+ \tau h    )  )
\big].
\end{equation}
As a result,
\begin{equation}
\label{Randomization-eq:estimate_one-step_weak_second_term}
\begin{aligned}
\big|
\E 
\big[
\mathbb{J}_1
\big]
\big|
=   
h
\Big|
\E 
\Big[
\E_\tau  
\big[
F   (   X  (t,x; t+ \tau h    )  )
\big]
- 
F (  X (t,x;t+ \tau h)
\Big]
\Big|
=0 .
\end{aligned}
\end{equation}
With regard to the third term, we use the H\"older inequality to deduce
\begin{equation}
\label{Randomization-eq:pri-estimate_one-step_weak_third_term0}
\big|
\E 
\big[
\mathbb{J}_2
\big]
\big| 
\leq  
\E 
\big[
\big|
\mathbb{J}_2
\big| 
\big]
\leq 
\Big(
\E 
\big[
\big|
\mathbb{J}_2
\big|^2
\big]
\Big)^\frac12.
\end{equation}
Again, by using  the H\"older inequality, together with \eqref{Randomization-eq:poly_growth_cond} and \eqref{Randomization-eq:|Tx|leqx},  one can obtain  
\begin{equation}
\label{Randomization-eq:pri-estimate_one-step_weak_third_term}
\begin{aligned}
\E 
\big[
\big|
\mathbb{J}_2
\big|^2
\big]
\leq   &
h^2
\E 
\Big[
\big|
F (  X (t,x;t+ \tau h)
)
-
F (
\Bar{Y}_m(t,x;t+ \tau h)
)
\big|^2
\Big]
\\ 
\leq   &
L_2^2  h^2
\E 
\bigg[
\Big( 1 
+
|  X (t,x;t+ \tau h)|^\gamma
+
|  \Bar{Y}_m(t,x;t+ \tau h)|^\gamma
\Big)^2
\Big|
X (t,x;t+ \tau h) 
-
\Bar{Y}_m (t,x;t+ \tau h)
\Big|^2
\bigg]
\\  
\leq   &
L_2^2  h^2
\underbrace{
\bigg(
\E 
\bigg[
\Big( 1 
+
|  X (t,x;t+ \tau h)|^\gamma
+
|   \Bar{Y}_m(t,x;t+ \tau h)|^\gamma
\Big)^4
\bigg]
\bigg)^{\frac{1}{2}}
}_{= :\mathbb{J}_{21}}
\\   &
\times 
\underbrace{
\bigg(
\E 
\bigg[
\Big|
X (t,x;t+ \tau h) 
-
\Bar{Y}_m (t,x;t+ \tau h)
\Big|^4
\bigg]
\bigg)^{\frac{1}{2}}
}_{= : \mathbb{J}_{22}}.
\end{aligned}
\end{equation}
Before treating $\mathbb{J}_{21}$ in \eqref{Randomization-eq:pri-estimate_one-step_weak_third_term}, in view of Lemma \ref{lem:properties_of_projected_operator_T} and \eqref{Randomization-eq:fun_ineq}, 
we first derive from \eqref{Randomization-eq:one_step_pRLMC} that, 
for any $q\ge 0$,
\begin{equation}
\label{Randomization-eq:bound-moments-Y_m}
\begin{aligned}
\E 
\Big[
\big|   \Bar{Y}_m(t,x;t+ \tau h)
\big|^{2q}
\Big]
\leq  &
3^{2q-1}\Big(
| x  |^{2q}
+
\ell^{2q} h^{2q}
\big|
F ( \cT^h(x) ) 
\big|^{2q}
+
2^q
\E 
\big[
| 
W_{t+\tau h }
-
W_t)
|^{2q}
\big]
\Big)
\\ 
\leq   &
3^{2q-1}\Big(
| x  |^{2q}
+
2^{2q-1}\ell^{2q} 
\big(
L_2'^{2q} 
+ 
2^{2q}
L_2'^{2q}
\big)
d^q
h^{q}
+
2^q
(2q-1)!!
d^q h^q
\big]
\Big)
\\ 
\leq   &
C^{\ell}
\big(
| x  |^{2q}
+
d^q
\big),
\end{aligned}
\end{equation}
where $C^{\ell}$ depends on $q, \ell, L_2', L_2$.
With this at hand,
utilizing \eqref{Randomization-eq:fun_ineq} and Proposition \ref{prop:uniform_in_time_bounded_moments_to_LSDE} gives
\begin{equation}
\label{Randomization-eq:mathbb-J-21}
\begin{aligned}
\mathbb{J}_{21}
\leq &
\bigg(
3^3
\Big(
1 
+
\E 
\big[
|  X (t,x;t+ \tau h)|^{4\gamma}
\big]
+
\E 
\big[
|   \Bar{Y}_m(t,x;t+ \tau h)|^{4\gamma}
\big]
\Big)
\bigg)^{\frac{1}{2}}
\\ 
\leq   &
\bigg(
3^3
\Big(
1 
+
|x|^{4\gamma}
+
\CM_1(2\gamma)
d^{2\gamma}
+
C^{\ell}
\big(
| x  |^{4\gamma}
+
d^{2\gamma}
\big)
\Big)
\bigg)^{\frac{1}{2}}
\\ 
\leq   &
C^{\ell}
\big( 
| x  |^{2\gamma}
+
d^{\gamma}
\big),
\end{aligned}
\end{equation}
where $C^{\ell}$ is a dimension-independent constant, depending on $\mu, \mu', \gamma, \ell, L_2',L_2$. To handle $\mathbb{J}_{22}$ in \eqref{Randomization-eq:pri-estimate_one-step_weak_third_term},
we combine \eqref{Randomization-eq:one_step_pRLMC} with
\begin{equation}
X (t,x;t+ \tau h)
=
x 
+
\int_{t}^{t+\tau h}
F   (   X  (t,x; s    )  )
\dd s
+
\sqrt{2}
(
W_{t+\tau h }
-
W_t) 
\end{equation}
to first infer 
\begin{equation}
\begin{aligned}
\mathbb{J}_{22}
=
\bigg(
\E 
\bigg[
\Big|
\int_{t}^{t+\tau h}
\big(
F   (   X  (t,x;s    )  )
-
\ell 
F (  \cT^h (x)  )
\big)
\dd s
\Big|^4
\bigg]
\bigg)^{\frac{1}{2}}.
\end{aligned}
\end{equation}
Further treatments of $\mathbb{J}_{22}$ are divided into two cases depending on
different choices of $\ell$. For the case $\ell= 0$, by \eqref{Randomization-eq:poly_growth_cond1}, the H\"older inequality and Proposition  \ref{prop:uniform_in_time_bounded_moments_to_LSDE}, one can find some constant $C:=C(\mu, \mu', L_2',L_2)$ such that 
\begin{equation}
\begin{aligned}
\mathbb{J}_{22}
\leq   &
h^{\frac{3}{2}}
\bigg(
\E 
\bigg[
\int_{t}^{t+\tau h}
\Big|
F   (   X  (t,x;s    )  )
\Big|^4
\dd s
\bigg]
\bigg)^{\frac{1}{2}}
\\  
\leq    &
h^{\frac{3}{2}}
\bigg(
\E_\tau 
\bigg[
\int_{t}^{t+\tau h}
\Big(
L_2'^4 d^2 
+
2^4 
L_2^4
\E_W 
\Big[
\big|  X  (t,x;s    )  
\big|^{4\gamma + 4}
\Big]
\Big)
\dd s
\bigg]
\bigg)^{\frac{1}{2}}
\\ 
\leq   &
h^{\frac{3}{2}}
\bigg(
\E_\tau 
\bigg[
\int_{t}^{t+\tau h}
\Big(
L_2'^4 d^2 
+
2^4 
L_2^4
\big(
| x  |^{4\gamma+4}   
+
\CM_1 (2\gamma+2)
d^{2\gamma+2}
\big) 
\Big)
\dd s
\bigg]
\bigg)^{\frac{1}{2}}
\\  
\leq   & 
C 
\big(  
|  x   |^{2\gamma+2  }+  d^{\gamma+1}
\big) 
h^2.
\end{aligned}
\end{equation}
For the case $\ell=1$, we employ \eqref{Randomization-eq:poly_growth_cond}, \eqref{Randomization-eq:fun_ineq}, the H\"older inequality, the Young inequality, Proposition \ref{prop:uniform_in_time_bounded_moments_to_LSDE} and Lemma \ref{lem:properties_of_projected_operator_T} to derive 
\begin{equation}
\begin{aligned}
\mathbb{J}_{22}
\leq   &   
\bigg(
\E 
\bigg[
\Big|
\int_{t}^{t+\tau h}
\big(
F   (   X  (t,x;s    )  )
-
F (  \cT^h (x)  )
\big)
\dd s
\Big|^4
\bigg]
\bigg)^{\frac{1}{2}}
\\
\leq   &
 h^{\frac{3}{2}}
\bigg(
\E_\tau 
\bigg[
\int_{t}^{t+\tau h}
\E_W 
\Big[
\big|
F   (   X  (t,x;s    )  )
-
F (  \cT^h (x)  )
\big|^4
\Big]
\dd s
\bigg]
\bigg)^{\frac{1}{2}}
\\ 
\leq    &
 h^{\frac{3}{2}}
\bigg(
\E_\tau 
\bigg[
\int_{t}^{t+\tau h}
L_2^4
\E_W 
\bigg[
\Big( 1 
+
|  X   (t,x;s    )|^\gamma
+
|  \cT^h   ( x )|^\gamma
\Big)^4
\Big|
X   (t,x;s  ) 
-
\cT^h (  x )
\Big|^4
\bigg]
\dd s
\bigg]
\bigg)^{\frac{1}{2}}
\\ 
\leq   &
  h^{\frac{3}{2}}
\bigg(
\E_\tau 
\bigg[
\int_{t}^{t+\tau h}
6^3  L_2^4
\E_W 
\bigg[
\Big( 1 
+
|  X   (t,x;s    )|^{4\gamma}
+
|  \cT^h   ( x )|^{4\gamma}
\Big)
\Big(
\big|
X   (t,x;s  ) 
\big|^4
+
\big|
\cT^h (  x )
\big|^4
\Big)
\bigg]
\dd s
\bigg]
\bigg)^{\frac{1}{2}}
\\ 
\leq     &
  h^{\frac{3}{2}}
\bigg(
\E_\tau 
\bigg[
\int_{t}^{t+\tau h}
6^3  L_2^4
\Big( 1 
+
\E_W 
\Big[
|  X   (t,x;s    )|^{4\gamma+4 }
\Big]
+
|   x |^{4\gamma+4 }
\Big)
\dd s
\bigg]
\bigg)^{\frac{1}{2}}
\\  
\leq    &
 h^{\frac{3}{2}}
\bigg(
\E_\tau 
\bigg[
\int_{t}^{t+\tau h}
6^3  L_2^4
\Big( 1 
+
|   x |^{4\gamma+4 }
+
\CM_1(2\gamma+2 ) 
d^{2\gamma +2 }
+
|   x |^{4\gamma+4 }
\Big)
\dd s
\bigg]
\bigg)^{\frac{1}{2}}
\\
\leq   &
C \big( 
|   x |^{2\gamma+2 }
+ 
d^{\gamma +1}
\big) h^2,
\end{aligned}
\end{equation}
where $C$ is a dimension-independent constant, depending on $\mu, \mu', \gamma,L_2$.
Collecting the cases $\ell=0$ and $\ell =1$ together, we show, for any $\ell \in \{0,1\}$ there exist some constants $C:= C(\mu, \mu', \gamma,\ell,L_2', L_2 )$ such that
\begin{equation}
\label{Randomization-eq:mathbb-J-22}
\mathbb{J}_{22}
\leq    
C \big( 
|   x |^{2\gamma+2 }
+ 
d^{\gamma +1}
\big) h^2.
\end{equation}
Inserting this and \eqref{Randomization-eq:mathbb-J-21} into \eqref{Randomization-eq:pri-estimate_one-step_weak_third_term} yields 
\begin{equation}
\label{Randomization-eq:estimate_onestep_strong_third-term}
\E 
\big[
\big|
\mathbb{J}_2
\big|^2
\big]
\leq   
C \big( 
|   x |^{4\gamma+2 }
+ 
d^{2\gamma +1}
\big) h^4,
\end{equation}
which further  implies 
\begin{equation}
\label{Randomization-eq:estimate_one-step_weak_third_term_of_pRLMC}
\big|\E 
\big[
\mathbb{J}_2
\big]\big|
\leq   
C \big( 
|   x |^{4\gamma+2 }
+ 
d^{2\gamma +1}
\big)^{\frac{1}{2}} h^2.
\end{equation}
Here $C$ is a dimension-independent constant, depending on $\mu, \mu', \gamma,\ell,L_2',L_2$.
Following the same arguments as used in  \eqref{Randomization-eq:pri-estimate_one-step_weak_third_term0} and \eqref{Randomization-eq:pri-estimate_one-step_weak_third_term},  one can treat $|\E [\mathbb{J}_3 ]|$ as follows:
\begin{equation}
\label{Randomization-eq:estimate_one-step_weak_fourth_term_of_pRLMC}
\begin{aligned}
|\E [\mathbb{J}_3 ]|
\leq &
\E \big[|\mathbb{J}_3 |\big]
\leq 
\Big(
\E 
\Big[
|  \mathbb{J}_3   |^2
\Big]
\Big)^{\frac{1}{2}}
\\  
\leq    &
h
\bigg(
\E 
\Big[
\big|
F (   
\Bar{Y}_m (t,x;t+ \tau h)
)
-
F (  
\cT^h (    \Bar{Y}_m (t,x;t+ \tau h))
)
\big|^2
\Big]
\bigg)^{\frac{1}{2}}
\\ 
\leq   &
L_2 h
\bigg(
\E 
\bigg[
\Big( 1 
+
|  \Bar{Y}_m (t,x;t+ \tau h)|^\gamma
+
|  \cT^h 
(   \Bar{Y}_m(t,x;t+ \tau h))|^\gamma
\Big)^2
\\    \times   &
\Big|
\Bar{Y}_m (t,x;t+ \tau h) 
-
\cT^h (   \Bar{Y}_m (t,x;t+ \tau h)
\Big|^2
\bigg]
\bigg)^{\frac{1}{2}}
\\  
\leq   &
2 L_2  h
\bigg(
\E 
\bigg[
\Big( 1 
+
2 |  \Bar{Y}_m (t,x;t+ \tau h)|^\gamma
\Big)^2
\vartheta^{-8\gamma-8}
d^{-4} h^4
\Big|
  \Bar{Y}_m (t,x;t+ \tau h))
\Big|^{8\gamma +10}
\bigg]
\bigg)^{\frac{1}{2}}
\\  
\leq   & 
C 
\bigg( 
d^{-4}
\E
\Big[
\big|
  \Bar{Y}_m (t,x;t+ \tau h))
\big|^{8\gamma +10}
\Big]
+
d^{-4}
\E
\Big[
\big|
  \Bar{Y}_m (t,x;t+ \tau h))
\big|^{10\gamma +10}
\Big]
\bigg)
h^3
\\ 
\leq  &
C \big(
d^{5\gamma+1}
+
d^{-4}
|x|^{10\gamma +10}
\big)^{\frac{1}{2}}
h^3, 
\end{aligned}
\end{equation}
where $C$ is a dimension-independent constant, depending on $\vartheta,\mu, \mu', \gamma,\ell,L_2',L_2$.
Putting estimates \eqref{Randomization-eq:estimate_one-step_weak_first_term}, \eqref{Randomization-eq:estimate_one-step_weak_second_term}, \eqref{Randomization-eq:estimate_one-step_weak_third_term_of_pRLMC}  and \eqref{Randomization-eq:estimate_one-step_weak_fourth_term_of_pRLMC} together and using the Young inequality, we derive from 
\eqref{Randomization-eq:pri-estimate-one-step-weak-error-pRLMC} that 
\begin{equation}
\big|
\E 
\big[
X(t,x;t+h)
-
\Bar{Y}(t,x;t+h)
\big]
\big|
\leq    
\Bar{K}_{1}^{\ell}
\Big(  
d^{5\gamma+1}
+
d^{-4}
|
x
|^{10\gamma+10}
\Big)^{\frac12}
h^2,
\end{equation}
where $\Bar{K}_1^{\ell}$ is a dimension-independent constant, depending on $\mu,\mu',\gamma,\vartheta,\ell,L_2,L_2'$.
Next we intend to obtain the one-step mean-square error bound. According to \eqref{Randomization-eq:difference_onestep_of_pRLMC}, one can use a fundamental inequality  to deduce
\begin{equation}
\begin{aligned}
& \E 
\Big[
\big|
X(t,x;t+h)
-
\Bar{Y}(t,x;t+h)
\big|^2
\Big]
\leq   
4 
\Big(
| x - \cT^h (x) |^2
+
\E 
\Big[
\big|
\mathbb{J}_1
\big|^2
\Big]
+
\E 
\Big[
\big|
\mathbb{J}_2
\big|^2
\Big]
+
\E 
\Big[
\big|
\mathbb{J}_3
\big|^2
\Big]
\Big).
\end{aligned}
\end{equation}
Owing to Lemma \ref{lem:properties_of_projected_operator_T}, one can easily see 
\begin{equation}
\label{Randomization-eq:estimate_onestep_strong_first_term}
|x - \cT^h(x)|^2
\leq  
4 \vartheta^{-8(\gamma+1)} d^{-4}
h^{4}
|  x |^{8\gamma+10}.
\end{equation}
Using the H\"older inequality and  Lemma \ref{lem:holder_continuous_of_LSDE_non_lip} yields 
\begin{equation}
\label{Randomization-eq:estimate_onestep_strong_sec-term}
\begin{aligned}
\E 
\Big[
\big|
\mathbb{J}_1
\big|^2
\Big]
\leq   
h 
\int_{t}^{t+h}
\E 
\Big[
\big|
F   (   X  (t,x;s   )  )
-
F (  X (t,x;t+ \tau h)
)
\big|^2
\Big]
\dd s
\leq   
C\big(
d^{2\gamma+1}
+
|
x
|^{4\gamma+2}
\big)
h^3.
\end{aligned}
\end{equation}
With the help of estimates \eqref{Randomization-eq:estimate_onestep_strong_third-term}, \eqref{Randomization-eq:estimate_one-step_weak_fourth_term_of_pRLMC}, \eqref{Randomization-eq:estimate_onestep_strong_first_term} and 
\eqref{Randomization-eq:estimate_onestep_strong_sec-term}, 
 we get
\begin{equation}
\E 
\Big[
\big|
X(t,x;t+h)
-
\Bar{Y}(t,x;t+h)
\big|^2
\Big]
\leq  
(\Bar{K}_2^{\ell})^2
\Big(  
d^{5\gamma+1}
+
d^{-4}
|
x 
|^{10\gamma+10}
\big]
\Big) h^3,
\end{equation}
where $\Bar{K}_2^{\ell}$ is a dimension-independent constant, depending on $\mu,\mu',\gamma,\vartheta,\ell,L_2,L_2'$. 
The proof is completed.}
\end{proof}

\noindent
\textit{Proof of Proposition \ref{prop:finite_time_error_analysis_pRLMC}.}
\textcolor{black}{In light of Theorem 3.3 of \cite{yang2025nonasymptotic}, one can combine   Assumptions 
\ref{Randomization-ass:diss_cond_and_grad_Lip_cond},
\ref{Randomization-ass:poly_growth_condition}, and Proposition \ref{prop:uniform_in_time_bounded_moments_to_pRLMC}, Lemma \ref{lem:onestep_waek_and_strong_error_of_pRLMC} to obtain
\begin{equation}
\E
\big[
|
X_{t_n} 
-
\Bar{Y}_n
|^2
\big]
\leq  
\exp
\big( (1 + 10L + 6 L_2)  T\big)
\Bar{K}^{\textcolor{black}{\ell}} 
\big(
 d^{(11\gamma+2)/2}
+
 d^{-4}
\E
\big[
|  x_0 |^{11\gamma+10}
\big]
\big)h^2,
\end{equation}
where $\Bar{K}^{\ell} := C (\Bar{K}_1^{\ell},\Bar{K}_2^{\ell}, \CM_1(\frac{11\gamma+10}{2}), \frac{\CM_3^\ell}{\mu})=
C(\mu,\mu',\gamma,\vartheta,\ell,L_2,L_2')$,
as required.}

\section{Numerical experiments}
\label{Randomization-sec:Numerical experiments}

In this section, we present several numerical experiments to validate  the non-asymptotic error bounds established above. Two types of target distributions are considered, including a Gaussian mixture distribution for the R(S)LMC algorithms, and a Gibbs distribution with double-well potential for the modified schemes pR(S)LMC.

The considered Gaussian mixture model consists of two symmetric Gaussian components with equal weights:
\begin{equation}
\pi_1 (x)  
= 
\tfrac12  \mathcal{N} (\mu_1,  I_d )  
+
\tfrac12  \mathcal{N} 
(-\mu_1, I_d ), 
\quad 
x \in \R^d,
\end{equation} 
where the mean vector is chosen as 
$\mu_1=\frac{2}{\sqrt{d}}(1,\cdots,1)^T \in \R^d$ so that $| \mu_1 | = 2$.
The Gibbs distribution with a double-well potential considered is given by
\begin{equation}
\pi_2(x)
= 
\frac{1}{Z}   \exp (- |x|^4 + |x|^2),   
\quad 
x \in  \R^d,
\end{equation}
where 
$
Z:= \int_{\R^d} \exp (- |x|^4 + |x|^2)  \, \dd x
$
is the normalization constant. 
%
%
It has been verified in Section 4 of \cite{yang2025nonasymptotic} that the Gaussian mixture model fulfills the assumptions of Theorem \ref{thm:main_thm_for_RLMC} and the Double-well potential adheres to those required in Theorem \ref{thm:main_thm_for_pRLMC}.

\textbf{Convergence rate.}
In the following experiments, we focus on target distributions in dimensions $d=10$ and $d=20$.
The "exact" solutions are identified as the numerical ones produced by the R(S)LMC and pR(S)LMC using a fine stepsize $h_{ref}= 2^{-15}$. Expectations are approximated by Monte Carlo averages over $2000$ independent trajectories.
To test the convergence rate, the schemes are implemented using various stepsizes $h\in \{2^{-6},\cdots, 2^{-11} \}$.
The root mean-square strong error is computed at the terminal time $T=2$ and plotted in Fig \ref{fig:GM_d=10} and \ref{fig:GM_d=20} against $h$ on a log–log scale, together with reference slopes of order $\tfrac12$ and order $1$. 
From Fig \ref{fig:GM_d=10} and \ref{fig:GM_d=20}, we observe that RLMC and RSLMC both exhibit an empirical convergence rate of order one. For pRLMC and pRSLMC applied to the Gibbs distribution with  double-well potential, one can observe the expected order-one convergence rate from Fig \ref{fig:DW_d=10} and \ref{fig:DW_d=20}.

\textbf{Dimension dependence.}
To examine the dimension dependence, we consider a range of dimensions $d \in \{ 40, 42, 44, 46, 48, 50 \}$.
For each dimension, we simulate $2000$ independent trajectories over the time interval $[0,2]$. The "exact" solutions are generated on a highly refined temporal grid, using the classical Euler–Maruyama with step size $h=2^{-12}$. 
For each method and each dimension, the root mean-square error is computed 
and plotted against the dimension on a log–log scale. A reference slope of $0.5$ is added to illustrate the empirical scaling rate. As shown in Fig \ref{fig:dd1} and \ref{fig:dd2}, the dimension dependence of RLMC and RSLMC under the  fixed stepsizes $h=2^{-8}$ and $h=2^{-9}$ is half-order.

\begin{figure}
\centering
\begin{minipage}{0.48\textwidth}
\centering
\includegraphics[width=\textwidth]{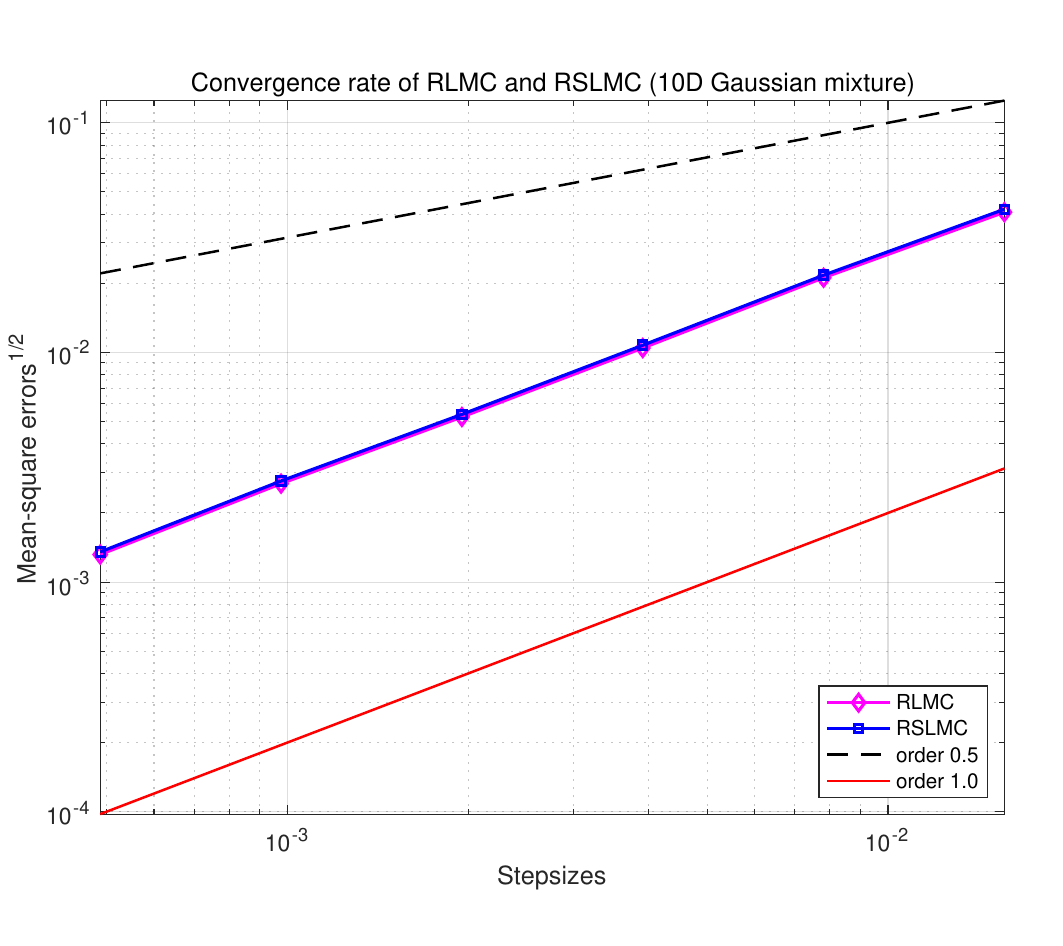}
\caption{Order-one convergence
of RLMC and RSLMC for 10-dimensional Gaussian Mixture.}
\label{fig:GM_d=10}
\end{minipage}
\hfill
\begin{minipage}{0.48\textwidth}
\centering
\includegraphics[width=\textwidth]{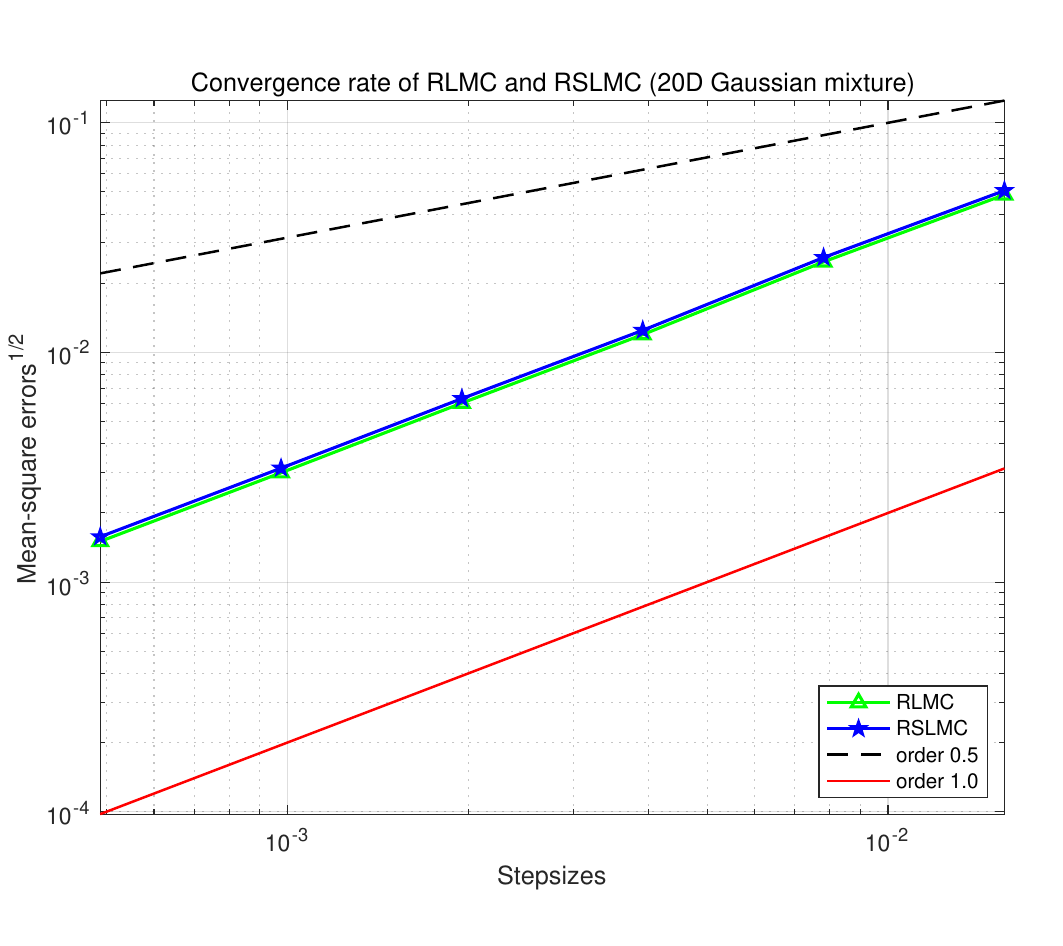}
\caption{Order-one convergence
of RLMC and RSLMC for 20-dimensional Gaussian Mixture.}
\label{fig:GM_d=20}
\end{minipage}
\end{figure}

\begin{figure}
\centering
\begin{minipage}{0.48\textwidth}
\centering
\includegraphics[width=\textwidth]{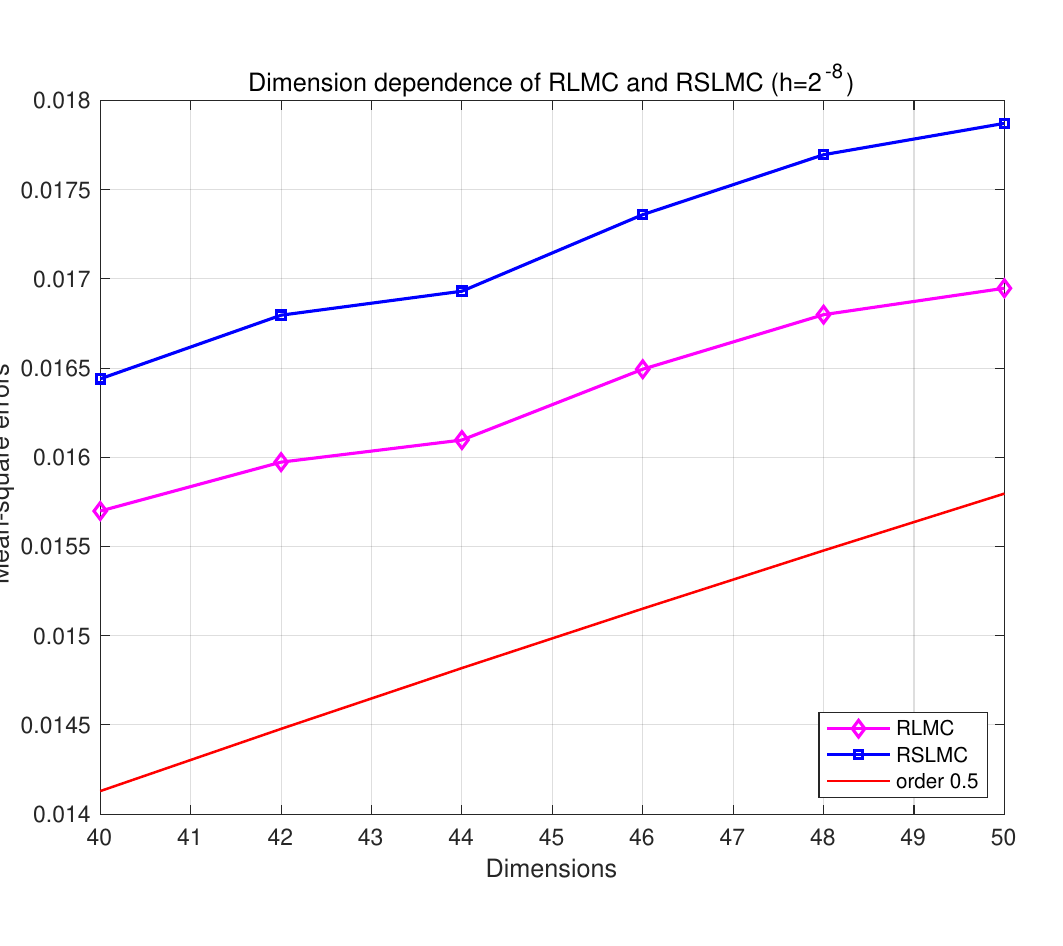}
\caption{$\sqrt{d}$ dimension dependence of RLMC and RSLMC ($h=2^{-8}$).}
\label{fig:dd1}
\end{minipage}
\hfill
\begin{minipage}{0.48\textwidth}
\centering
\includegraphics[width=\textwidth]{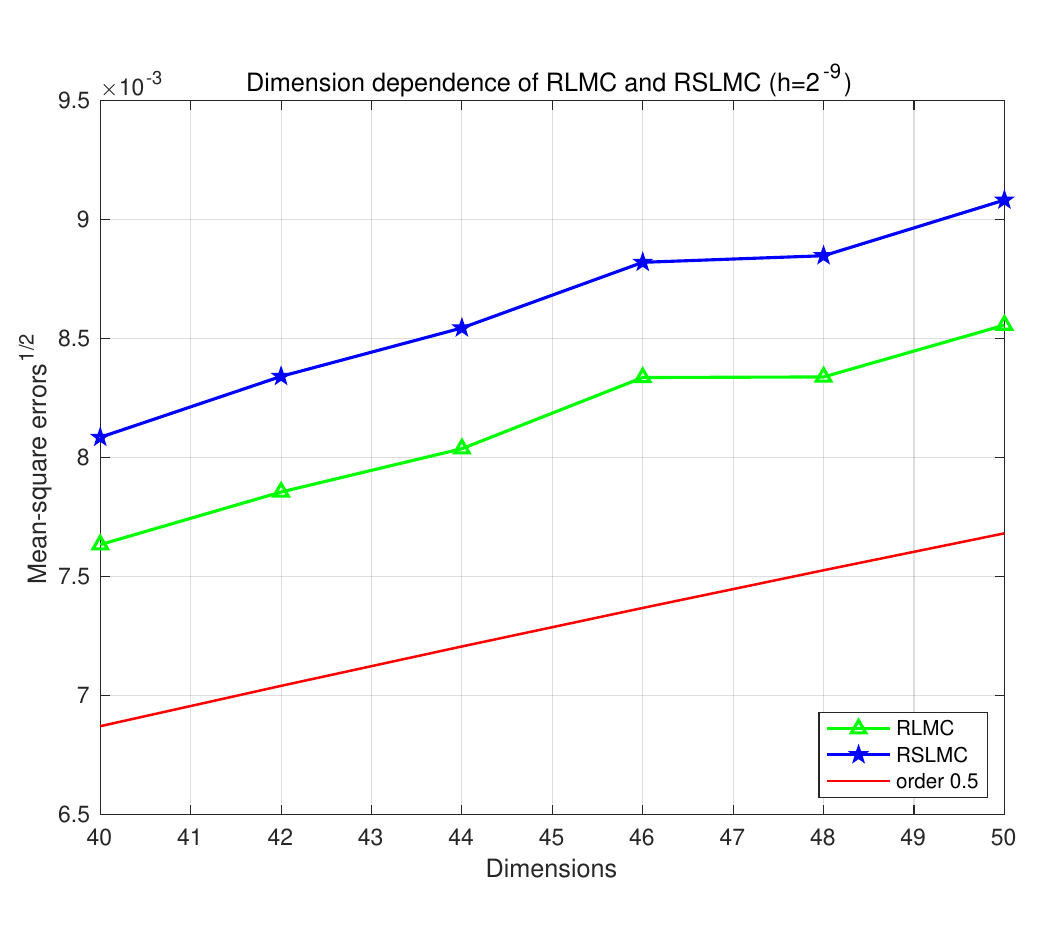}
\caption{$\sqrt{d}$ dimension dependence of RLMC and RSLMC ($h=2^{-9}$)}
\label{fig:dd2}
\end{minipage}
\end{figure}

\begin{figure}
\centering
\begin{minipage}{0.48\textwidth}
\centering
\includegraphics[width=\textwidth]{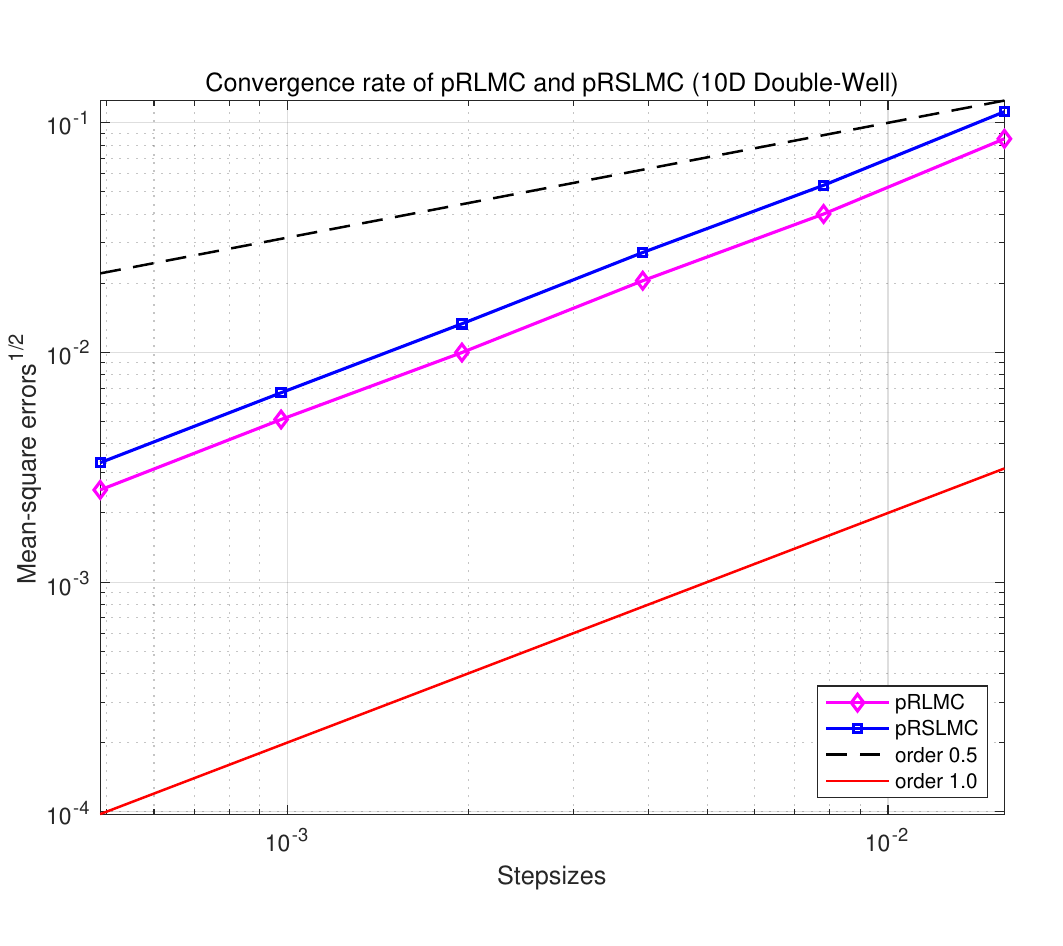}
\caption{Order-one convergence
of pRLMC and pRSLMC for 10-dimensional Double well.}
\label{fig:DW_d=10}
\end{minipage}
\hfill
\begin{minipage}{0.48\textwidth}
\centering
\includegraphics[width=\textwidth]{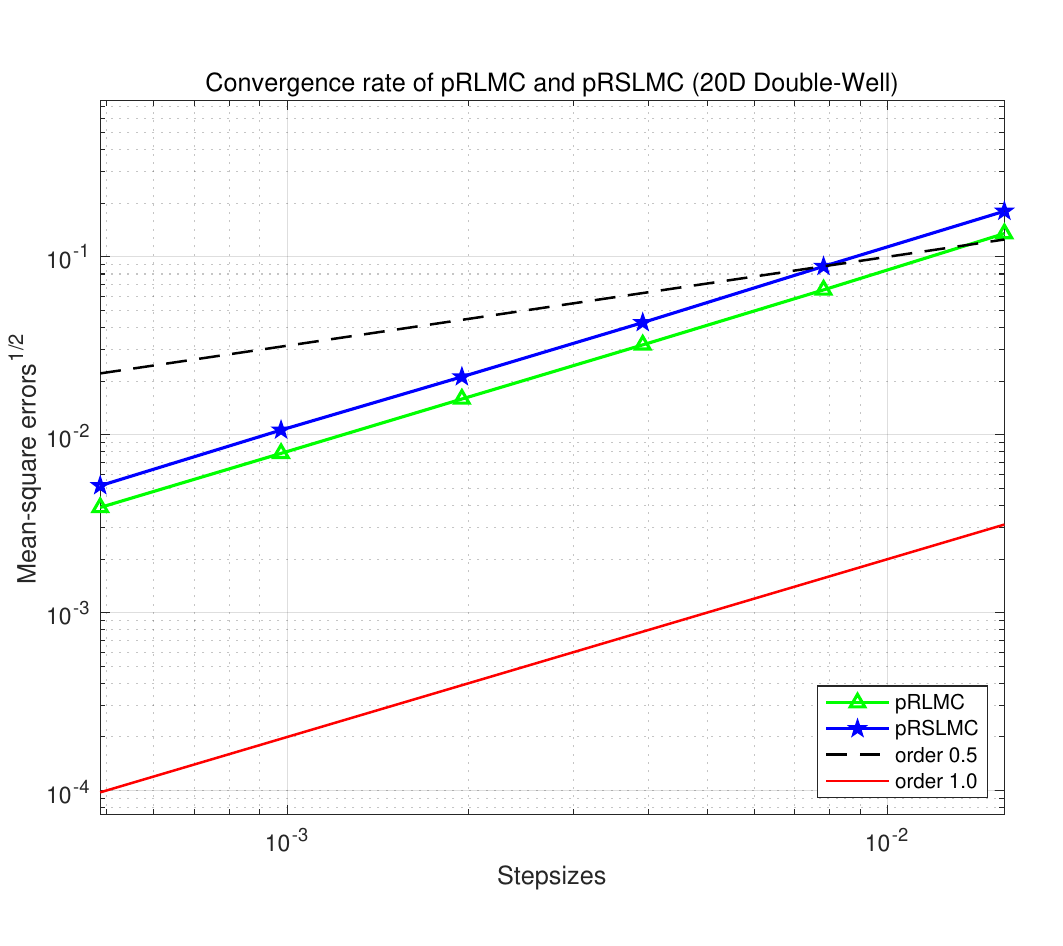}
\caption{Order-one convergence
of pRLMC and pRSLMC for 20-dimensional Double well.}
\label{fig:DW_d=20}
\end{minipage}
\end{figure}

\section{Conclusion and future work}
\label{Randomization-sec:Conclusion}
In this work, we propose a new kind of randomized splitting Langevin Monte Carlo (RSLMC) algorithm for sampling from high dimensional distributions, which  requires only one evaluation of $\nabla U$ per one time step and thus improves the existing randomized Langevin Monte Carlo (RLMC) in terms of the number of evaluations of $\nabla U$. Moreover, we establish a non-asymptotic error bound $O(\sqrt{d} h)$ in $\mathcal{W}_2$-distance for the existing RLMC and the newly proposed RSLMC in the framework of LSI, without requiring additional smoothness assumptions on $U$ other than the gradient Lipschitz condition. Also, we examine the sampling problem when the gradient of the potential $U$ is non-globally Lipschitz with superlinear growth. In this case, we propose a modified RSLMC sampler and derive a non-asymptotic error bound in $\mathcal{W}_2$-distance with convergence rates and dimension dependence revealed. 
The key idea of the non-asymptotic error analysis in
the non-convex setting is to acquire the desired uniform-in-time convergence via finite-time convergence 
combined with the exponential ergodicity of SDEs and uniform-in-time moment bounds
of algorithms.
We highlight that this approach of error analysis also applies to higher order LMC sampling algorithms \cite{li2019stochastic,sabanis2019higher} and sampling based on the underdamped Langevin dynamics \cite{yu2024langevin,shen2019randomized}, which are our ongoing works  \cite{Lv2025,wang2025Accelerating}.
In the future, we also intend to follow this idea and investigate the non-asymptotic error bound in other distances under other weaker functional inequalities, such as Poincare inequality \cite{chewi2024analysis}.

\bibliography{main}

\end{document}